\documentclass{article}

\usepackage{microtype}
\usepackage{graphicx}
\usepackage{subfigure}
\usepackage{booktabs} 

\usepackage{hyperref}

\usepackage[accepted]{icml2025}

\usepackage{preamble}

\icmltitlerunning{Simplicity bias and optimization threshold in ReLU networks}

\begin{document}
\setlength{\abovedisplayskip}{5pt}
\setlength{\topsep}{0.5em}

\setcounter{tocdepth}{3}

\doparttoc 
\faketableofcontents 

\twocolumn[
\icmltitle{Simplicity Bias and Optimization Threshold in Two-Layer ReLU Networks}

\begin{icmlauthorlist}
\icmlauthor{Etienne Boursier}{orsay}
\icmlauthor{Nicolas Flammarion}{epfl}
\end{icmlauthorlist}

\icmlaffiliation{orsay}{INRIA, LMO, Université Paris-Saclay, Orsay, France}
\icmlaffiliation{epfl}{TML Lab, EPFL, Switzerland}

\icmlcorrespondingauthor{Etienne Boursier}{etienne.boursier@inria.fr}

\icmlkeywords{Neural Networks, Simplicty Bias, Implicit Bias, One hidden ReLU Network, Early Alignment}

\vskip 0.3in
]



\printAffiliationsAndNotice{}  


\begin{abstract}\looseness=-1
Understanding generalization of overparametrized models remains a fundamental challenge in machine learning. 
The literature mostly studies generalization from an interpolation point of view, taking convergence towards a global minimum of the training loss for granted. 
This interpolation paradigm does not seem valid for complex tasks such as in-context learning or diffusion. It has instead been empirically observed that the trained models go from global minima to spurious local minima of the training loss as the number of training samples becomes larger than some level we call \textit{optimization threshold}. 
This paper explores theoretically this phenomenon in the context of two-layer ReLU networks. We demonstrate that, despite overparametrization, networks might converge towards simpler solutions rather than interpolating training data, which leads to a drastic improvement on the test loss. 
Our analysis relies on the so called early alignment phase, during which neurons align toward specific directions. This directional alignment 
leads to a simplicity bias, wherein the network approximates the ground truth model without converging to the global minimum of the training loss. Our results suggest this bias, resulting in an optimization threshold from which interpolation is not reached anymore, is beneficial and enhances the generalization of trained models.
\end{abstract}
\vspace{-1em}
\section{Introduction}

Understanding the generalization capabilities of neural networks remains a fundamental open question in machine learning~\citep{zhang2021understanding,10.5555/3295222.3295344}. 
Traditionally, research has focused on explaining why neural networks models can achieve zero training loss while still generalizing well to unseen data in supervised learning tasks~\citep{chizat2018global,mei2018mean,rotskoff2022,chizat2020implicit,
boursier2022gradient,boursier2023penalising}. This phenomenon is often attributed to overparametrization enabling models to find solutions that interpolate the training data yet avoid overfitting~\citep{belkin2019reconciling,Bartlett_Montanari_Rakhlin_2021}.
However, the advent of generative AI paradigms—such as large language models~\citep{vaswani2017attention} and diffusion models~\citep{dhariwal2021diffusion}—has introduced a paradigm shift in our understanding of generalization. In these settings, models can generate new data and perform novel tasks without necessarily interpolating the training data, raising fresh questions about how and why they generalize. 
This shift can be illustrated by two seemingly unrelated applications: in-context learning with transformers and generative modeling using diffusion methods.

Firstly, in-context learning (ICL) refers to the ability of large pretrained transformer models to learn new tasks from just a few examples, without any parameter updates~\citep{brown2020language,min2022rethinking}. A central question is whether ICL enables models to learn tasks significantly different from those encountered during pretraining. While prior work suggests that ICL leverages mechanisms akin to Bayesian inference~\citep{xie2022an,garg2022can,bai2023transformers}, the limited diversity of tasks in pretraining datasets may constrain the model's ability to generalize. \citet{raventos2024pretraining}  investigated this effect by focusing on regression problems to quantify how increasing the variety of tasks during pretraining affects ICL's capacity to generalize to new, unseen tasks, in context.

Secondly, diffusion models have made remarkable strides in generating high-quality images from high-dimensional datasets~\citep{pmlr-v37-sohl-dickstein15,NEURIPS2020_4c5bcfec,song2021scorebased}. These models learn to generate new samples by training denoisers to estimate the score function—the gradient of the log probability density—of the noisy data distribution~\citep{song2019generative}. A significant challenge in this context is approximating a continuous density from a relatively small training set without succumbing to the curse of dimensionality. Although deep neural networks may tend to memorize training data when the dataset is small relative to the network's capacity~\citep{somepalli2023diffusion,carlini2023extracting}, \citet{yoon2023diffusion,kadkhodaie2023generalization} observed they generalize well when trained on sufficiently large datasets, rendering the model's behavior nearly independent of the specific training set.

\looseness=-1
The common thread connecting these examples is a fundamental change in how gradient descent behaves in overparametrized models when the number of data points exceeds a certain threshold. Rather than converging to the global minimum of the training loss, gradient descent converges to a simpler solution closely related to the true loss minimizer. In learning scenarios involving noisy data, the most effective solutions are often those that do not interpolate the data. Despite their capacity to overfit, these models exhibit a simplicity bias, generalizing well to the underlying ground truth instead of merely fitting the noise in the training data. While simplicity bias generally refers to the tendency of models to learn features of increasing complexity, until reaching data interpolation~\citep{arpit2017closer,rahaman2019spectral,kalimeris2019sgd,huh2021low}; this phenomenon seems to stop before full interpolation with modern architectures \citep[even when training for a very long time, see e.g.][Figure 4]{raventos2024pretraining}. This observation underscores a significant shift in our understanding and approach to generalization in machine learning.

\looseness=-1
In this paper, we theoretically investigate this phenomenon in the toy setting of shallow ReLU networks applied to a regression problem.
While multilayer perceptrons are foundational elements shared by the aforementioned models, focusing on shallow networks remains a significant simplification with respect to the architectures and algorithms used for training transformers and diffusion models. 
Despite this simplification, we aim to gain theoretical insights that shed light on similar behaviors observed in more complex models. 
%
Some recent theoretical works argued that overparametrized networks do not necessarily converge to global minima. In particular, \citet{qiao2024stable} showed this effect for unidimensional data by illustrating the instability of global minima. \citet{boursier2024early} advanced a different reason for this effect, given by the \textit{early alignment} phenomenon: when initialized with sufficiently small weights, neurons primarily adjust their directions rather than their magnitudes in the early phase of training, aligning along specific directions determined by the stationary points of a certain known function.

\paragraph{Contributions.}
Our first contribution is to show that this function driving the early alignment phase concentrates around its expectation, which corresponds to the true loss function, as the number of training samples grows large. For simple teacher architectures, this expected function possesses only a few critical points. As a result, after the early alignment phase, the neurons become concentrated in a few key directions associated with the ground truth model. This behavior reveals a simplicity bias at the initial stages of training.
Moreover, this directional concentration is believed to contribute to the non-convergence to the global minimizer of the training loss. However, this characterization only pertains to the initial stage of training. Therefore we extend our analysis to provide, under a restricted data model, a comprehensive characterization of the training dynamics, demonstrating that the simplicity bias persists until the end of training when the number of training samples exceeds some optimization threshold. We here provide an informal version of our main theorem, corresponding to \cref{thm:noconvergence}.
\begin{thm}[Informal]\label{thm:informal} Consider a specific regression setting with $n$ data samples of dimension $d$. Using a two-layer ReLU neural network trained with gradient flow and small random initialization, we show that if $n \gtrsim d^3 \log d$, then regardless of the network's width, the learned function closely approximates the ordinary least squares solution.
\end{thm}
In particular, \cref{thm:informal} shows that in the overparametrized regime, the final estimator does not minimize the training loss globally, yet it achieves near-optimal performance on the test data. 
We then confirm empirically our predictions.

\section{Preliminaries}\label{sec:prelim}

This section introduces the setting and the early alignment phenomenon, following the notations and definitions of \citet{boursier2024early}.
\subsection{Notations}
\looseness=-1
We denote by $\bS_{d-1}$ the unit sphere of $\R^d$ and $B(\mathbf{0},1)$ the unit ball. We note $f(t)=\bigO[p]{g(t)}$, if there exists a constant $C_p$, that only depends on $p$ such that for any $t$, $|f(t)|\leq C_p g(t)$. We drop the $p$ index, if the constant $C_p$ is universal and does not depend on any parameter. Similarly we note $f(t)= \Omega_p(g(t))$, if there exists a constant $C_p>0$, that only depends on $p$ such that $f(t)\geq C_p g(t)$ and we write $f(t)= \Theta_p(g(t))$ if both $f(t)=\bigO[p]{g(t)}$ and $f(t)= \Omega_p(g(t))$. 
For any bounded set $A$, $\cU(A)$ denotes the uniform probability distribution on the set $A$. 
\subsection{Setting}
\looseness=-1 We consider $n$ data points $(x_k,y_k)_{k\in[n]} \in \R^{d+1}$ drawn i.i.d. from a distribution $\mu\in\cP(\R^{d+1})$. We also denote by $\bX=[x_1^{\top},\ldots,x_n^{\top}]\in\R^{d\times n}$ and $\by=(y_1,\ldots,y_n)\in\R^{n}$ respectively the matrix whose columns are given by the input vectors $x_k$ and the vector with coordinates given by the labels $y_k$. 
A two layer ReLU network is parametrized by $\theta=(w_j,a_j)_{j\in[m]}\in\R^{m\times(d+1)}$, corresponding to the prediction function
\begin{equation*}
h_{\theta}:x\mapsto \textstyle\sum_{j=1}^m a_j \sigma(w_j^{\top}x),
\end{equation*}
where $\sigma$ is the ReLU activation given by $\sigma(z)=\max(0,z)$. 
While training, we aim at minimizing the empirical square loss over the training data, defined as
\begin{equation*}
L(\theta ; \bX,\by) = \frac{1}{2n}\textstyle\sum_{k=1}^n (h_{\theta}(x_k)-y_k)^2.
\end{equation*}
As the limiting dynamics of (stochastic) gradient descent with vanishing learning rates, we study a subgradient flow of the training loss, which satisfies for almost any $t\in\R_+$,
\begin{equation}\label{eq:gradientflow}
\dot\theta(t) \in -\partial_{\theta}L(\theta(t); \bX, \by),
\end{equation}
where $\partial_{\theta}L$ stands for the Clarke subdifferential of $L$ w.r.t. $\theta$.
\subsection{Early alignment dynamics}
\paragraph{Initialization.} \looseness=-1
In accordance to the feature learning regime \citep{chizat2019lazy}, 
the $m$ neurons of the neural network are initialized as
\begin{equation}\label{eq:init1}
(a_j(0), w_j(0))=\lambda \,  m^{-1/2}\,(\ta_j,\tw_j),
\end{equation}
\looseness=-1
where $\lambda>0$ is the scale of initialization and $(\ta_j,\tw_j)$ are vectors drawn i.i.d. from some distribution, satisfying the following domination property for any $m\in\N$:
\begin{equation}\label{eq:domination}\begin{gathered}
|\ta_j|\geq \|\tw_j\| \text{ for any } j\in[m] \ \text{ and }\ 
\frac{1}{m}\textstyle\sum_{j=1}^m \ta_j^2 \leq 1.
\end{gathered}
\end{equation}
\looseness=-1
This property is common and allows for a simpler analysis, as it ensures that the signs of the output neurons $a_j(t)$ remain unchanged while training \citep{boursier2022gradient}.

\paragraph{Neuron dynamics.}

In the case of two layer neural networks with square loss and ReLU activation, \cref{eq:gradientflow} can be written for each neuron $i\in[m]$ as
\begin{equation}\label{eq:ODE}
\begin{gathered}
\dot w_i(t) \in a_i(t)\D_n(w_i(t),{\theta}(t)) \\
\dot a_i(t) = {w}_i(t)^{\top}D_n({w}_i(t),{\theta}(t))\rangle,
\end{gathered}
\end{equation}
where the vector $D_n({w}_i(t),{\theta}(t))$ and set $\D_n({w}_i(t),{\theta}(t))$ are defined as follows, with $\partial\sigma$ the subdifferential of the ReLU activation $\sigma$:
\begin{gather*}
D_n({w},{\theta}) = \frac{1}{n}\sum_{k=1}^n \iind{{x}_k^{\top}{w}>0}(y_k-h_{{\theta}}({x}_k))x_k,\\
\D_n(w,\theta)\! = \!\left\{ \!\frac{1}{n}\!\sum_{k=1}^n \!\eta_k(y_k-h_{\theta(t)}(x_k))x_k  \big| \eta_k \!\in\! \partial\sigma( {x}_k^{\top} {w})\!\right\}\!.
\end{gather*}
These derivations directly follow from the subdifferential of the training loss. 
In particular, $D_n({w},{\theta})$ corresponds to a specific vector (subgradient) in the subdifferential $\D_n({w},{\theta})$. Also observe that the set $\D_n({w},\theta)$ depends on ${w}$ only through its activations $A_n({w})$, defined as
\begin{equation*}
A_n: \begin{array}{l} \R^d \to \{-1,0,1\}^{n} \\ {w} \mapsto \sign({w}^{\top}{x}_k)_{k\in[n]}\end{array}.
\end{equation*}
\looseness=-1
Furthermore, $\D_n({w},{\theta})$  only depends on $\theta$  via the prediction function $h_{{\theta}}$. 
This observation is crucial to the early alignment phenomenon. 

\paragraph{Early alignment.}

In the small initialization regime described by \cref{eq:init1}, numerous works highlight an early alignment phase in the initial stage of training~\citep{maennel2018gradient,atanasov2022neural,boursier2024early,kumar2024directional,tsoy2024simplicity}. During this phase, the neurons exhibit minimal changes in norm, while undergoing significant changes in direction. This phenomenon is due to a discrepancy in the derivatives of the neurons' norms (which scale with $\lambda$) and of their directions (which scale in $\Theta(1)$). 
Specifically,  for a sufficiently small initialization scale $\lambda$, the neurons align towards the critical directions of the following function $G_n$ defined as
\begin{equation}\label{eq:G}
G_n: w \mapsto w^{\top} D_n(w,\mathbf{0}).
\end{equation}
\looseness=-1
$G_n$ is continuous, piecewise linear and can be interpreted as the correlation between the gradient information around the origin (given by $D_n(w,\mathbf{0})$) and the neuron $w$. 
The network neurons thus align with the critical directions on the sphere of $G_n$ during the early training dynamics. These critical directions are called \textit{extremal vectors}, defined as follows.
\begin{defin}\label{def:extremal}
A vector $D\in \R^d$
is \textbf{extremal} with respect to $G_n$ if there exists $w\in\bS_{d-1}$ such that both hold
\begin{enumerate}[parsep=0pt,itemsep=0pt,topsep=0pt,partopsep=0pt]
\item $D\in\D_n(w, \mathbf{0})$;
\item $D= \mathbf{0} \text{ or } A_n(D) \in \{A_n(w), -A_n(w)\}$.
\end{enumerate}
\end{defin}
This definition directly follows from the KKT conditions of the maximization (or minimization) problem, constrained on the sphere, of the function $G_n$.

\paragraph{Implications of early alignment.}
\looseness=-1%
By the end of the early alignment, most if not all neurons are nearly aligned with some extremal vector $D$. \citet{maennel2018gradient,boursier2024early} argue that only a few extremal vectors exist in typical learning models. We further explore this claim in \cref{sec:aligngeom}. 
As a consequence, only a few directions are represented by the network's weights at the end of the early dynamics, even though the neurons  cover all possible directions at initialization. 
\citet{boursier2024early} even show that this \textit{quantization of directions} can prevent the network from interpolating the training set at convergence despite the overparametrization of the network.

\looseness=-1
Although this \textit{failure of interpolation} has been considered a drawback by \citet{boursier2024early}, we show in  \cref{sec:noconvergence} that it can also lead to a beneficial phenomenon of simplicity bias. 
Specifically, \cref{sec:noconvergence} illustrates on a simple linear example that for a large number of training samples, the model does not converge to interpolation. Instead, it converges towards the ordinary least square (OLS) estimator of the data.
As a consequence, the model fits the true signal of the data, while effectively ignoring label noise. 
Before studying this example, we must first  understand how  extremal vectors behave as the number of training samples increases.
\vspace{-0.5em}
\section{Geometry of alignment with many samples}\label{sec:aligngeom}

\looseness=-1
We here aim to describe the geometry of the function $G_n$, with a specific focus on the extremal vectors, as the number of training samples $n$ becomes large. 
These vectors are key in driving the early alignment phase of the training, making them essential to understanding the initial dynamics of the parameters. 
%

This section provides a general result on the concentration of gradient information~$D_n$ of the train loss and support that the early alignment behavior in the infinite-data setting does not differ significantly from that in the large but finite $n$ case. While the tail bound version of \cref{thm:Dconcentration} is central to our analysis in \cref{sec:noconvergence}, the results of \Cref{sec:aligngeom} are not only stated in a general form for broader applicability, but also constitute standalone contributions that may be useful in future work.

Despite non-smoothness of the loss (due to ReLU activations), we can leverage the piecewise constant structure of the vector function 
$D_n(w)$, along with typical Rademacher complexity arguments, to derive uniform concentration bounds on the random function $w\mapsto D_n(w)$. 
\begin{restatable}{thm}{Dconc}\label{thm:Dconcentration}
If the marginal law of $x_1$ is continuous with respect to the Lebesgue measure, then for any $n\in\N$,
\begin{multline*}
\E_{\bX,\by}\big[\sup_{w\in\bS_{d-1}}\sup_{D_n\in\D_n(w,\mathbf{0})}\|D_n - D(w)\|_2\big]=\\ \mathcal{O}\Big(\sqrt{\frac{d\log n}{n}\E[\|y_1 x_1\|^2_2]}\Big),
\end{multline*}
where for any $w\in\bS_{d-1}$, $D(w)=\mathbb{E}[\iind{w^{\top}x_1>0}y_1 x_1]$.
\end{restatable}
\cref{thm:Dconcentration} indicates that as $n$ grows large, the sets $\D_n(w,\mathbf{0})$ converge to the corresponding vectors for the true loss, given by $D(w)$, at a rate $\sqrt{\frac{d\log n}{n}}$. 
Moreover this rate holds uniformly across all possible directions of $\R^d$ in expectation. 
A probability tail bound version of \cref{thm:Dconcentration}, which bounds this deviation with high probability, can also be derived (see \cref{thm:Dtailbound} in \cref{app:tailbound}). A complete proof of \cref{thm:Dconcentration} is provided in \cref{sec:Dconcproof}.

\looseness=-1
When $n\to\infty$, the alignment dynamics are thus driven by vectors $D_n$ which are close to their expected value $D(w)$. Furthermore, when $n\to\infty$, the activations of a weight $A_n(w)$ exactly determine the direction of this weight, as every possible direction is then covered by the training inputs $x_k$.  Specifically, for an infinite dataset indexed by $\N$, whose support covers all directions of $\R^d$, and defining the infinite activation function $A$ as
\begin{gather*}
A: \begin{array}{l} \bS_{d-1} \to \{-1,0,1\}^{\N} \\w \mapsto \sign(w^{\top}x_k)_{k\in\N}\end{array};
\end{gather*}
then $A$ is injective. In this infinite data limit, the functions $G_n$ converge to the function $G:w\mapsto w^\top D(w)$, which is differentiable in this limit, and a vector $D\in \R^d$ is \textbf{extremal} with respect to $G$ if there exists $w\in\bS_{d-1}$ such that both 
\begin{equation}\label{eq:extremeinf}
\text{1. } D=D(w)\qquad
\text{2. } D= \mathbf{0} \text{ or } \frac{D}{\|D\|_2} \in \{ w, -w\}.
\end{equation}
When $n$ becomes large, the extremal vectors of the data then concentrate toward the 
vectors satisfying \cref{eq:extremeinf}. This is precisely quantified by \cref{prop:Dconcentration} below.
\begin{restatable}{prop}{Dcoro}\label{prop:Dconcentration}
Assume the marginal law of $x_1$ is continuous w.r.t. the Lebesgue measure and that $\E[\|x_1 y_1\|^4]<\infty$. 
Then for any $\varepsilon>0$, there is $\nstar(\varepsilon)=\bigO[\varepsilon,\mu]{d \log d}$ such that for any $n\geq\nstar(\varepsilon)$, with probability at least $1-\bigO[\mu]{\frac{1}{n}}$:
for any extremal vector $D_n$ of the finite data $(\bX,\by)\in\R^{n\times(d+1)}$, there exists a vector $D^{\star} \in\R^d$ satisfying \cref{eq:extremeinf}, such that 
\begin{equation*}
\|D_n-D^{\star}\|_2 \leq \varepsilon.
\end{equation*}
\end{restatable}
\cref{prop:Dconcentration} states that for large $n$, the extremal vectors concentrate towards the vectors satisfying \cref{eq:extremeinf}. Note that \cref{prop:Dconcentration} is not a direct corollary of \cref{thm:Dconcentration}, but its proof relies on the tail bound version of \cref{thm:Dconcentration} and continuity arguments. A complete proof is given in \cref{sec:proofDconcentration}. 

\paragraph{Early alignment towards a few directions.}
\looseness=-1
Besides laying the ground for \cref{thm:noconvergence}, \cref{prop:Dconcentration} aims at describing the geometry of the early alignment when the number of training samples grows large. In particular, \cref{prop:Dconcentration} shows that all extremal vectors concentrate towards the directions satisfying \cref{eq:extremeinf}. Although such a description remains abstract, we believe it is satisfied by only a few directions for many data distributions. As an example, for symmetric data distributions, it is respectively satisfied by a single or two directions, when considering a one neuron or linear teacher. 
More generally, we conjecture it should be satisfied by a small number of directions as soon as the labels are given by a small teacher network. Proving such a result is yet left for future work. 

The early alignment phenomenon has been described in many works, to show that after the early training dynamics, only a few directions (given by the extremal vectors) are represented by the neurons \citep{bui2021inductive,lyu2021gradient,boursier2022gradient,chistikov2023learning,min2024early,boursier2024early,tsoy2024simplicity}. However, these works all rely on specific data examples, where extremal vectors can be easily expressed for a finite number of samples. \cref{prop:Dconcentration} aims at providing a more general result, showing that for large~$n$, it is sufficient to consider the directions satisfying \cref{eq:extremeinf}, which is easier to characterize from a statistical perspective. 
We thus believe that \cref{prop:Dconcentration} advances our understanding of how sparse is the network representation (in directions) at the end of early alignment.


\cref{prop:Dconcentration} implies that for large values of $n$ ($\gtrsim d$), the early alignment phase results in the formation of a small number of neuron clusters, effectively making the neural network equivalent to a small-width network. Empirically, these clusters appear to be mostly 
preserved throughout training. The neural network then remains equivalent to a small-width network  along its entire training  trajectory.

\looseness=-1 In contrast, when the number of data is limited ($n\lesssim d$), this guarantee no longer holds and a large number of extremal vectors may exist. For example in the case of orthogonal data (which only holds for $n\leq d$), there are $\Theta(2^n)$ extremal vectors \citep{boursier2022gradient}. 
In such cases, there would still be a large number of neuron clusters at the end of the early alignment phase, maintaining a large \textit{effective width} of the network. Studying how this effective width is maintained until the end of training in the orthogonal case remains an 
open problem. We conjecture that for a mild overparametrization ($n \lesssim m\ll 2^n$),\footnote{\citet{boursier2022gradient} proved an effective width of $2$ at the end of training when $m\gtrsim 2^n$.} we would still have a relatively large effective width (increasing with $n$) at the end of training.
\section{Optimization threshold and simplicity bias}\label{sec:noconvergence}
\looseness=-1
The goal of this section is to illustrate the transition from interpolating the training data to a nearly optimal estimator (with respect to the true loss) that can arise when increasing the size of training data. Toward this end, this section proves  on a toy data example, that for a large enough number of training samples, an overparametrized network will not converge to a global minimum of the training loss, but will instead be close to the minimizer of the true loss. This is done by analyzing the complete training dynamics, whose first phase--the so-called early alignment phase--is controlled using the tail bound version of \cref{thm:Dconcentration}. 
To this end, we consider the specific case of a linear data model:
\begin{equation}\label{eq:linearmodel}
y_k = x_k^{\top}\betastar + \eta_k \quad \text{for any }k\in[n],
\end{equation}
where $\eta_k$ is some noise, drawn i.i.d. from a centered distribution.
We also introduce a specific set of assumptions regarding the data distribution.
\begin{assumption}\label{ass:noconvergence}
The samples $x_k$ and the noise $\eta_k$ are  drawn i.i.d. from distributions $\mu_{X}$, and $\mu_{\eta}$ satisfying, for some $c>0$:
\begin{enumerate}[parsep=0pt,itemsep=5pt,topsep=0pt,partopsep=0pt]
\item $\mu_{X}$ is symmetric, i.e., $x_k$ and $-x_k$ follow the same distribution;
\item $\mu$ is continuous with respect to the Lebesgue measure;
\item $\bP_{x\sim\mu_X}\left( |x^{\top}\betastar|\leq c\frac{\|x\|_2}{\sqrt{d}} \right) = 0$;
\item $\|\E_{x\sim\mu_X}[xx^{\top}]-\bI_d\|_{\op} < \min\left(\frac{c}{2\sqrt{d}\|\betastar\|_2}, \frac{3}{5} \right)$;
\item The random vector $x_k$ is $1$ sub-Gaussian and the noise satisfies $\E[\eta^4]<\infty$.
\end{enumerate}
\end{assumption}
Conditions 1, 2 and 5 in \cref{ass:noconvergence} are relatively mild. However, item 3 is quite restrictive: it is needed to ensure that the volume of the activation cone containing $\betastar$ does not vanish when $n\to\infty$. A similar assumption is considered by \citet{chistikov2023learning,tsoy2024simplicity}, for similar reasons. 
Additionally, Condition 4 ensures that $\E_{x}[xx^{\top}]\betastar$ and $\betastar$ are in the same activation cone. This assumption allows the training dynamics to remain within a single cone after the early alignment phase, significantly simplifying our analysis.

As an example, if the samples $x_k$ are distributed i.i.d. as
\begin{gather*}
x_k = \s_k\frac{\betastar}{\|\betastar\|} + \sqrt{d-1}\mathsf{v}_k 
\text{ with } \mathsf{v}_k\sim\cU(\bS_{d-1}\cap \{\beta^{\star}\}^\perp)\\
\text{and }\s_k \sim \cU\left([-1-\varepsilon,-1+\varepsilon]\cup[1-\varepsilon,1+\varepsilon]\right) ,
\end{gather*}
for a small enough $\varepsilon>0$ and $\mu_{\eta}$ a standard Gaussian distribution, then \cref{ass:noconvergence} is satisfied. 
In this section, we also consider the following specific initialization scheme for any $i\in[m]$:
\begin{equation}\label{eq:dominatedinit}
\begin{gathered}
w_i(0)\! \sim \!0.5 \lambda\, m^{-1/2}\, \cU(B(\mathbf{0},1)) \\\text{and }
a_i(0)\!\sim\!\lambda\, m^{-1/2}\, \cU(\{-1,1\}).
\end{gathered}
\end{equation}
\looseness=-1
In addition to the regime considered in \cref{eq:init1,eq:domination}, this initialization introduces  a stronger domination condition, as $|a_i(0)|\geq 2\|w_i(0)\|$. 
This condition reinforces the early alignment phase, ensuring that \textbf{all}  neurons are nearly aligned with extremal vectors by the end of this phase. \cref{ass:noconvergence,eq:dominatedinit} are primarily introduced to enable a tractable analysis and are discussed further in \cref{sec:discussion}. 
%

\looseness=-1
This set of assumptions allows to study the training dynamics separately on the following partition of the data:
\begin{gather*}
\cS_+ = \{ k\in[n] \mid x_k^{\top}\betastar \geq 0 \} \quad \text{and}\quad
\cS_- = [n]\setminus\cS_+ .
\end{gather*}
Hereafter, we denote by $\bX_{+}\in\R^{d\times|\cS_+|}$ (resp. $\bX_-$), the matrix with columns given by the vectors $x_k$ for $k\in\cS_+$ (resp. $k\in\cS_-$). Similarly, we denote by $\bY_{+}\in\R^{|\cS_+|}$ (resp. $\bY_{-}$) the vector with coordinates given by the labels $y_k$ for $k\in\cS_+$ (resp. $k\in\cS_-$).

\looseness=-1
Studying separately positive ($a_i>0$) and negative ($a_i<0$)  neurons, we prove \cref{thm:noconvergence} below, which states that at convergence for a large enough number of training samples, the sum of the positive (resp. negative) neurons correspond to the OLS estimator on the subset $\cS_+$ (resp. $\cS_-$).
\begin{restatable}{thm}{noconv}\label{thm:noconvergence}
If \cref{ass:noconvergence} holds and the initialization scheme follows \cref{eq:dominatedinit}, then there exists $\lambdastar=\Theta(\frac{1}{d})$ and $\nstar=\Theta(d^3\log d)$ such that for any $\lambda\leq\lambdastar$, any $m\in\N$ and $n\geq\nstar$, with probability $1-\bigO{\frac{d^2}{n}+\frac{1}{2^m}}$, the parameters $\theta(t)$ converge to some $\theta_{\infty}$ such that
\begin{equation*}
h_{\theta_{\infty}}(x) = (\beta_{n,+}^{\top} x)_+ - (-\beta_{n,-}^{\top} x)_+
\end{equation*}
for any $x\in\supp(\mu_X)$, 
where $\supp(\mu_X)$ is the support of the distribution $\mu_X$, $\beta_{n,+} = (\bX_+\bX_+^{\top})^{-1}\bX_{+}\bY_{+}$ and $\beta_{n,-} = (\bX_-\bX_-^{\top})^{-1}\bX_{-}\bY_{-}$ are the OLS estimator respectively on the data in $\cS_+$ and $\cS_-$\,.
\end{restatable}
Precisely, the estimator learnt at convergence for a large enough $n$ behaves $\mu_X$-everywhere as the difference of two ReLU neurons, with nearly opposite directions (thanks to the distribution symmetry), resulting in a nearly linear estimator. 
These directions correspond to the OLS estimator of the data in $\cS_+$ and in $\cS_-$, respectively. 
The complete proof of \cref{thm:noconvergence} is deferred to \cref{sec:noconvergenceproof}. We provide a detailed sketch in \cref{sec:sketch} below and discuss further \cref{thm:noconvergence} in \cref{sec:discussion}. 
\subsection{Sketch of Proof of \cref{thm:noconvergence}}\label{sec:sketch}

The proof of \cref{thm:noconvergence} examines the complete training dynamics of positive neurons ($a_i(0)>0$) and negative ones ($a_i(0)<0$) separately. This decoupling is possible at the end of the early phase, due to \cref{ass:noconvergence}, and is handled thanks to \cref{lemma:autonomous} in the Appendix.

First note that for the given model, there are only two vectors satisfying \cref{eq:extremeinf}, corresponding to $\frac{1}{2}\Sigma\betastar$ and $-\frac{1}{2}\Sigma\betastar$ respectively, for $\Sigma=\E_{x\sim\mu_X}[xx^\top]$. From then and thanks to the third point of \cref{ass:noconvergence}, the results from \cref{sec:aligngeom} imply that, for a large value of $n$ and with high probability, there are only two extremal vectors, both of which are close to the expected ones mentioned above.
By analyzing the early alignment phase similarly to \citet{boursier2024early}, we  show that by the end of this early phase, (i) all neurons have small norms; (ii) positive (resp. negative) neurons are aligned with $\Sigma\betastar$ (resp. $-\Sigma\betastar$). More specifically, at time $\tau$, defined as the end of the early alignment phase, we show that
\begin{equation*}
\forall i\in[m], \ \frac{w_i(\tau)}{a_i(\tau)}^{\top} \Sigma\betastar = \|\Sigma\betastar\| - \mathcal{O}\big(\lambda^{\varepsilon}+\sqrt{\frac{d^2\log n}{n}}\big).
\end{equation*}
From that point onward, all positive neurons are nearly aligned and behave as a single neuron until the end of training. Moreover, they remain in the same activation cone until the end of training. Namely for any $i\in[m]$ and $t\geq\tau$,
\begin{gather*}
a_i(t)\ x_k^{\top} w_i(t) > 0 \quad \text{for any } k\in\cS_+,\\
a_i(t)\ x_k^{\top}  w_i(t)  < 0 \quad \text{for any } k\in\cS_-.
\end{gather*}
We then show that during a second phase, all positive neurons grow until they reach the OLS estimator on the data in $\cS_+$. Mathematically, for some time $\tau_{2,+}>\tau$,
\begin{equation*}
\textstyle\sum_{i, a_i(0)>0}a_i(\tau_{2,+})w_i(\tau_{2,+}) \approx \beta_{n,+}.
\end{equation*}
Similarly, negative neurons end up close to $\beta_{n,-}$ after a different time $\tau_{2,-}$.  Proving this second phase is quite technical and is actually decomposed into a slow growth and fast growth phases, following a similar approach to \citet{lyu2021gradient,tsoy2024simplicity}.

At the end of the second phase, the estimation function is already close to the one described in \cref{thm:noconvergence}. From then, we control the neurons using a local Polyak-Łojasiewicz inequality (see Equation~\eqref{eq:PL}) to show that they remain close to their value at the end of the second phase, and actually converge to a local minimum corresponding to the estimation function $h_{\theta_\infty}$ described in \cref{thm:noconvergence}.
\subsection{Discussion}\label{sec:discussion}
\begin{figure*}[tb]
\centering
\begin{tikzpicture}[domain=0:12, xscale=1.2, yscale=0.7]
    \draw[->, thick] (0,0) -- (12,0) node[right] {$n$};
      \draw[->, thick] (0,0) -- (0,3) node[midway, left, xshift=-0pt] {\shortstack{excess\\  risk}};
      
    \node[vert] at (1.25,-0.25) {$n \ll d$};
    \node[blue] at (7.25,-0.25) {$n \gg d$};
        \node[red] at (10.75,-0.25) {$n \gg m$};
 	\node[vert] at (1.5,1.5) {\footnotesize\shortstack{\textbf{tempered}\\ \textbf{overfitting}}};
 	\node[blue] at (7.25,1.5) {\footnotesize\shortstack{\textbf{convergence to}\\ \textbf{spurious statio-}\\\textbf{nary point that}\\ \textbf{nicely generalizes}}};
 	\node[red] at (10.75,1.5) {\footnotesize\shortstack{\textbf{underparametrized}\\\textbf{regime: ERM}\\\textbf{generalizes well}}};

%

	\fill[vert, opacity=0.1] (0,0) rectangle (2.5,2.75);
	\fill[blue, opacity=0.1] (5,0) rectangle (9.5,2.75);
	\fill[red, opacity=0.1] (9.5,0) rectangle (12,2.75);
	
	\draw[blue, thick, line width=2pt, opacity=0.8] (5, 0) -- (5, 2.75);
	
	\draw[<-, thick, blue] (5, 1) --(4, 1.5) node[above, xshift=-10pt] {\footnotesize \shortstack{optimization\\threshold}};
	
    \draw[thick] plot[smooth] coordinates {(0,2.5)  (1,2.45) (2,2.2) (3,1.5) (4,0.5) (5,0.15) (6,0.1) (7,0.05) (8,0.02) (11,0.02)};

\end{tikzpicture}
\caption{\label{fig:regimes}Different regimes of generalization: in green ($n\ll d$), the trained estimator interpolates the data and leads to tempered overfitting; after the optimization threshold in blue ($n\gg d$), we converge to a spurious stationary point of the training loss which generalizes well despite overparametrization, this regime is our main focus; in the underparametrized regime in red ($m\gg n$), the global minima do not interpolate anymore and generalize well.}
\vspace{-1em}
\end{figure*}
\cref{thm:noconvergence} shows that, under a specific linear data model, when the number of training samples exceeds a certain \textit{optimization threshold}, the learned function converges to the OLS estimator—even in highly overparametrized settings where $m \gg n$. This result highlights two key insights:
\begin{itemize}[topsep=0pt,itemsep=0pt,leftmargin=10pt]
\item Despite overparametrization, the network can converge to a suboptimal solution of the training loss when initialized at a small scale.
\item This \textit{training failure} can in fact be beneficial: although suboptimal for the training loss, the resulting estimator is optimal for the test loss.
\end{itemize}

We now discuss further on \cref{thm:noconvergence} and its limitations.

\paragraph{Absence of interpolation.}
For many years, the literature has argued in favor of the fact that, if overparametrized enough, neural networks do converge towards interpolation of the training set, i.e., to a global minimum of the loss \citep{jacot2018neural,du2019gradient,chizat2018global,wojtowytsch2020convergence}. 

Yet, some recent works argued in the opposite direction that convergence towards global minima might not be achieved for regression tasks, even with infinitely overparametrized networks \citep{qiao2024stable,boursier2024early}. Indeed, \cref{thm:noconvergence} still holds as $m\to\infty$: although interpolation of the data is possible from a statistical aspect\footnote{Although the absence of bias term in the parametrization limits the expressivity of the neural network, interpolation is still possible as long as the data $x_i$ are pairwise non-proportional \citep[][Theorem 2]{carvalho2024positivity}.}, interpolation does not occur for optimization reasons. 
In this direction, \citet{qiao2024stable} claim that for large values of $n$ and univariate data, interpolation cannot happen because of the large (i.e., finite) stepsizes  used for gradient descent. Following \citet{boursier2024early}, we here provide a complementary reason, which is due to the early alignment phenomenon and loss of omnidirectionality of the weights (i.e., the fact that the weights represent all directions in $\R^d$). Note that this loss of omnidirectionality is specific to the (leaky) ReLU activation and does not hold for smooth activations \citep[see e.g.][Lemma C.10]{chizat2018global}.
We experimentally confirm in \cref{app:stable} that both visions are complementary, as interpolation still does not happen for arbitrarily small learning rates.

\paragraph{Simplicity bias.}
\looseness=-1
Simplicity bias has been extensively studied in the literature \citep{arpit2017closer,rahaman2019spectral,kalimeris2019sgd,huh2021low}. It is often described as the fact that networks learn features of increasing complexity while learning. In other words, simpler features are first learnt (e.g., a linear estimator), and more complex features might be learnt later. This has been observed in many empirical studies, leading to improved performance in generalization, except from a few nuanced cases \citep{shah2020pitfalls}. 
Yet in all these studies, the network interpolates the training set after being trained for a long enough time. In consequence, simplicity bias has been characterized by a first \textit{feature learning phase}; and is then followed by an \textit{interpolating phase}, where the remaining noise is fitted \citep{kalimeris2019sgd}.

\looseness=-1
We here go further by showing that this last interpolating phase does not even happen in some cases. \cref{thm:noconvergence} indeed claims that after the first feature learning phase, where the network learns a linear estimator, nothing happens in training. The interpolating phase never starts, no matter how long we wait for. While interpolation is often observed for classification problems in practice, it is generally much harder to reach for regression problems \citep{stewart2022regression,yoon2023diffusion,kadkhodaie2023generalization,raventos2024pretraining}. \cref{thm:noconvergence} confirms this tendency by illustrating a regression example where interpolation does not happen at convergence. 
Notably, we here focus on the blue regime in \cref{fig:regimes} and show that while the global minima poorly generalize in this regime, the optimization scheme only converges to a spurious local minimum of the training loss, which has much better generalization properties. This in stark contrast to the underparametrized regime -- $n\gg m$, in red in \cref{fig:regimes} -- where the global minimum has good guarantees, thanks to classical generalization bounds arguments \citep{bartlett2002rademacher}.

Although implicit bias and simplicity bias often refer to the same behavior in the literature, we here distinguish the two terms: implicit bias is generally considered in the regime of interpolation \citep{soudry2018implicit,lyu2019gradient,chizat2020implicit,ji2019implicit}, while simplicity bias still exists in absence of interpolation.

\paragraph{Improved test loss, due to overparametrization threshold.}
\looseness=-1
\cref{thm:noconvergence} states that for a large enough number of training samples, the interpolating phase does not happen during training, and the estimator then resembles the OLS estimator of the training set. In that regime, the excess risk scales as $\bigO{d/n}$ \citep{hsu2011analysis} and thus quickly decreases to~$0$ as the number of training samples grows. 
In contrast when interpolation happens, we either observe a \textit{tempered overfitting},  where the excess risk does not go down to $0$ as the number of samples grows \citep{mallinar2022benign}; or even a \textit{catastrophic overfitting}, where the excess risk instead diverges to infinity as the size of the training set increases \citep{joshi2023noisy}. 

\looseness=-1
The fact that the excess risk goes down to $0$ as $n$ grows in our example of \cref{sec:noconvergence} could not be due to a benign overfitting \citep{belkin2018overfitting,bartlett2020benign}, as benign overfitting occurs when the dimension $d$ also grows to infinity. We here consider a fixed dimension instead, and this reduced risk is then solely due to the \textit{optimization threshold}, i.e., the fact that for a large enough~$n$, the interpolating phase does not happen anymore. 
While some works rely on early stopping before this interpolating phase to guarantee such an improved excess risk \citep{ji2021early,mallinar2022benign,frei2023random}, it can be guaranteed without any early stopping after this optimization threshold. 
A similar threshold has been empirically observed in diffusion and in-context learning \citep{yoon2023diffusion,kadkhodaie2023generalization,raventos2024pretraining}, where the trained model goes from interpolation to generalization as the number of training samples increases.

\begin{figure*}[htbp] 
\centering
\subfigure[Evolution of train loss.]{\includegraphics[width=0.45\linewidth, trim=0.4cm 0.5cm 0.4cm 0.4cm, clip]{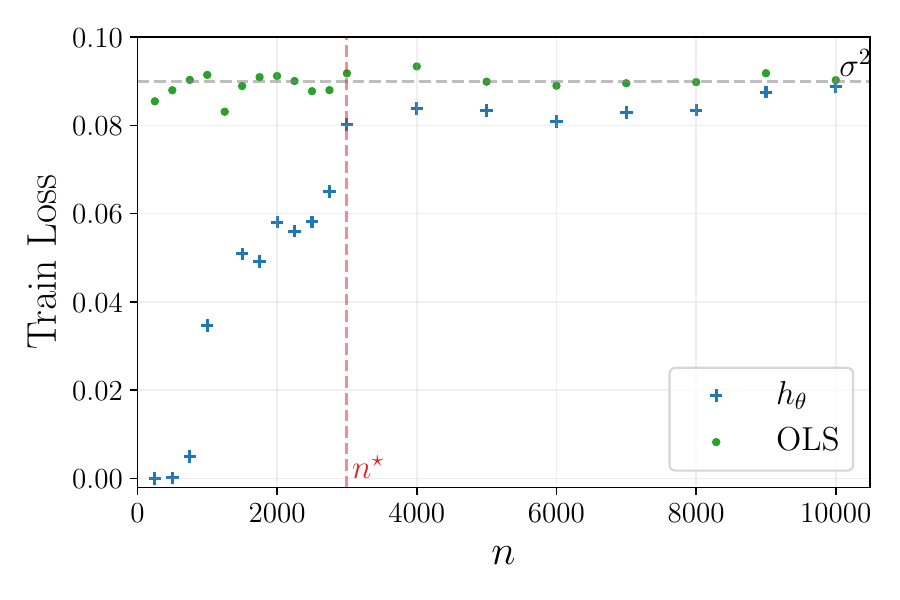}\label{fig:train}}
\hfill
\subfigure[Evolution of test loss.]{\includegraphics[width=0.45\linewidth, trim=0.4cm 0.5cm 0.4cm 0.4cm, clip]{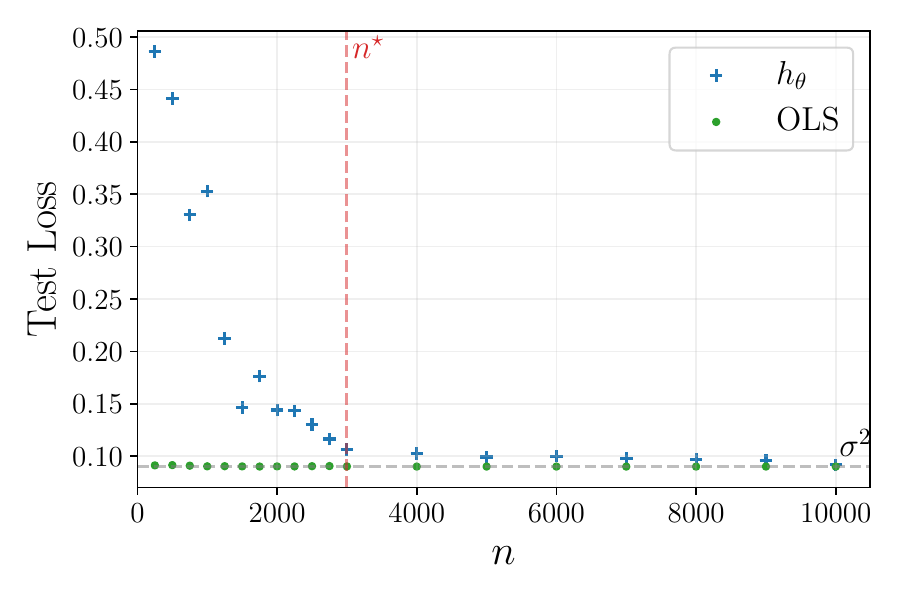}\label{fig:test}}
  \caption{\label{fig:mainexpe}Evolution of both train and test losses at convergence with respect to the number of training samples. $\sigma^2$ corresponds to the noise variance $\E[\eta^2]$.}
  \vspace{-1em}
\end{figure*}
\paragraph{Limitations and generality.}
\looseness=-1
While \cref{thm:noconvergence} considers a very specific setting, it describes a more general behavior. Although condition 3 of \cref{ass:noconvergence} and the initialization scheme of \cref{eq:dominatedinit} are quite artificial, they are merely required to allow a tractable analysis. The experiments of \cref{sec:expe} are indeed run without these conditions and yield results similar to the predictions of \cref{thm:noconvergence} for large enough $n$.

More particularly, condition 3 of \cref{ass:noconvergence} is required to ensure that only two extremal vectors exist. Without this condition, there could be additional extremal vectors, but all concentrated around these two main extremal ones. On the other hand, \cref{eq:dominatedinit} is required to enforce the early alignment phase, so that all neurons are aligned towards extremal vectors at its end. With a more general initialization, some neurons could move arbitrarily slowly in the early alignment dynamics, ending unaligned at the end of early phase. Yet, such neurons would be very rare. 
Relaxing these two assumptions would make the final convergence point slightly more complex than the one in \cref{thm:noconvergence}. Besides the two main ReLU components described in \cref{thm:noconvergence}, a few small components could also be added to the final estimator, without significantly changing the reached excess risk, as observed in \cref{sec:expe}. This is observed in \cref{fig:mainexpe}, where the training loss is only slightly smaller than the training loss of OLS, with a comparable test loss.

From a higher level, \cref{thm:noconvergence} is restricted to a linear teacher and a simple network architecture. It remains hard to assess how well the considered setting reflects the behavior of more complex architectures encountered in practice. 
We believe that the different conclusions of our work remain valid in more complex setups. In particular, additional experiments in \cref{app:expe} run with a more complex teacher, GeLU activations or with Adam optimizer yield similar behaviors: the obtained estimator does not interpolate for a large number of training samples, but instead accurately approximates the minimizer of the test loss. 
Similar behaviors have also been observed on more complex tasks as generative modeling or in-context learning \citep{yoon2023diffusion,kadkhodaie2023generalization,raventos2024pretraining}. Despite overparametrization, the trained model goes from perfect interpolation to generalization, as it fails at interpolating for a large number of training samples. In these works as well, this absence of interpolation does not seem due to an early stopping, but rather to convergence to a local minimum \citep[see e.g.,][Figure~4]{raventos2024pretraining}. 

\looseness=-1 Lastly, \cref{thm:noconvergence} requires a very large number of samples with respect to the dimension, i.e., $n\gtrsim d^3\log d$. Our experiments confirm that the optimization threshold only appears for a very large number of training samples with respect to the dimension. 
However, similar behaviors seem to occur for smaller orders of magnitude for $n$ in more complex learning problems, such as the training of diffusion models \citep{yoon2023diffusion,kadkhodaie2023generalization}. This dependency in $d$ might indeed be different for more complex architectures (e.g., with attention) and is worth investigating for future work.

\section{Experiments}\label{sec:expe}
\begin{figure*}[t] 
\centering
\subfigure[$n=500$.]{\includegraphics[width=0.43\linewidth, trim=0.4cm 0.5cm 0.4cm 0.4cm, clip]{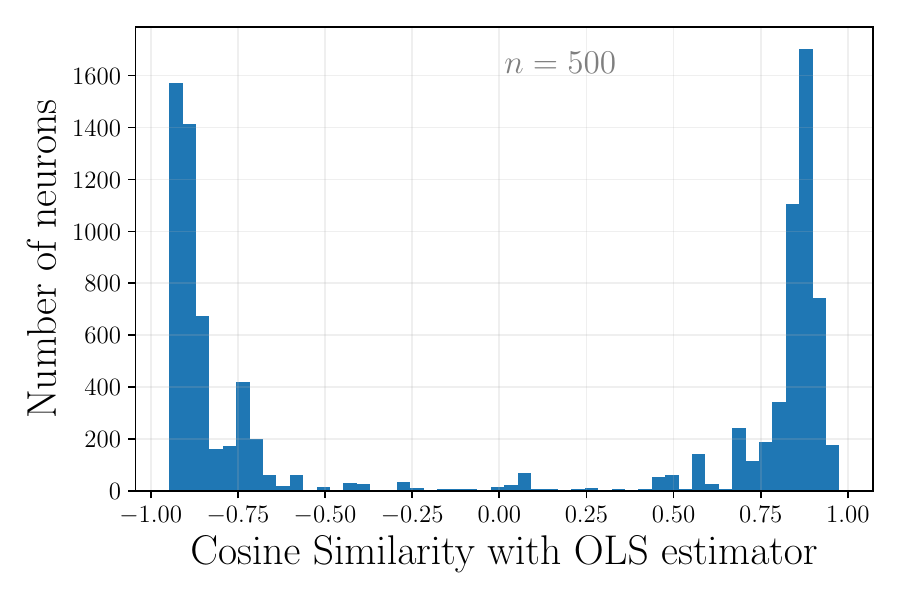}\label{fig:cosine1}}
\hfill
\subfigure[$n=5\ 000$.]{\includegraphics[width=0.43\linewidth, trim=0.4cm 0.5cm 0.4cm 0.4cm, clip]{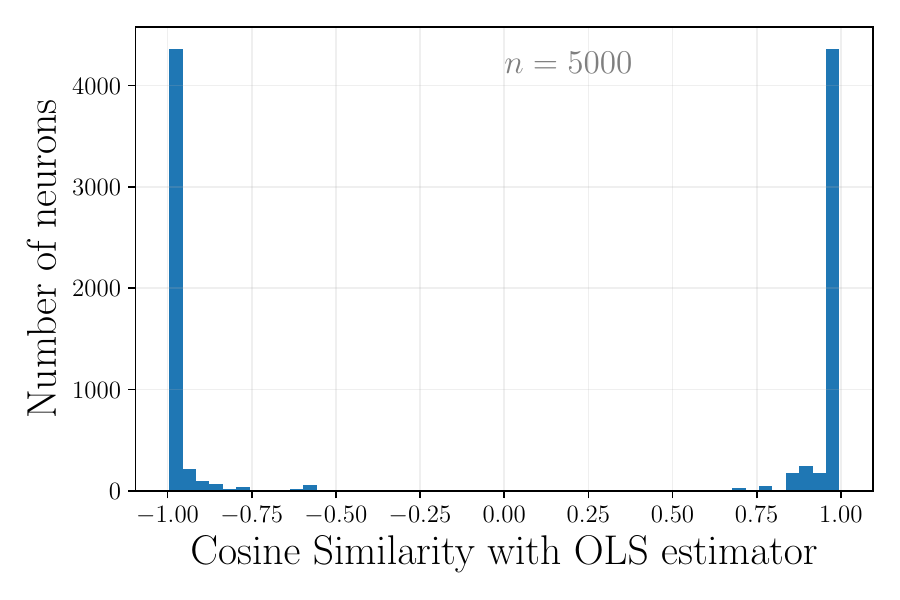}\label{fig:cosine2}}
  \caption{\label{fig:cosine}Histogram of the cosine similarities of the neurons with the true OLS estimator $\hbeta$, at the end of training.}
\end{figure*}
This section illustrates our results on experiments on a toy model close to the setting of \cref{sec:noconvergence}. More precisely, we train overparametrized two-layer neural networks ($m=10\ 000$) until convergence, on data from the linear model of \cref{eq:linearmodel}. The network is trained via stochastic gradient descent and the dimension is fixed to $d=5$ to allow reasonable running times. The setup here is more general than \cref{sec:noconvergence}, since i) the data input $x_k$ are drawn from a standard Gaussian distribution (which does not satisfy \cref{ass:noconvergence}); ii) the neurons are initialized as centered Gaussian of variance $10^{-5}/m$ (which does not satisfy \cref{eq:domination,eq:dominatedinit}). We refer to \cref{app:expe} for details on the considered experiments and additional experiments. 

\cref{fig:mainexpe} illustrates the behavior of both train loss and test loss at convergence, when the size of the training set $n$ varies. As predicted by \cref{thm:noconvergence}, when $n$ exceeds some optimization threshold, the estimator at convergence does not interpolate the training set. Instead, it resembles the optimal OLS estimator, which yields a test loss close to the noise level $\E[\eta^2]$. In contrast for smaller training sets, the final estimator interpolates the data at convergence, which yields a much larger test loss than OLS, corresponding to the tempered overfitting regime \citep{mallinar2022benign}.

\looseness=-1
This optimization threshold is here located around $\nstar=3000$, which suggests that the large dependency of this threshold in the dimension (which is here $5$) in \cref{thm:noconvergence} seems necessary--see \cref{app:dim} for experiments with larger dimensions.
%
We still observe a few differences here with the predictions of \cref{thm:noconvergence}, which are due to the two differences in the setups mentioned above. Indeed, even after this optimization threshold, the test loss of the obtained network is slightly larger than the one of OLS, while \cref{thm:noconvergence} predicts they should coincide. This is because in the experimental setup, a few neurons remain disaligned with the extremal ones at the end of the early alignment phase. These neurons will then later in training grow in norm, trying to fit a few data points. However there are only a few of such neurons, whose impact thus becomes limited--see \cref{fig:cosine}. As a consequence, they only manage to slightly improve the train loss and have little impact on the test loss.

\subsection{Cosine similarity with OLS estimator}\label{app:cosine}
\looseness=-1
To illustrate \cref{thm:noconvergence} and the fact that neurons end up aligned with the OLS estimator beyond the optimization threshold $\nstar$, \cref{fig:cosine} shows histograms of the cosine similarities\footnote{The cosine similarity between two vectors $u,v\in\R^d$ is defined as $\cos(u,v)=\frac{u^{\top}v}{\|u\|\ \|v\|}$.} between all the neurons $w_i$ of the network at the end of training and the true OLS estimator $\hbeta = (\bX\bX^{\top})^{-1}\bX\bY$, for different sample complexities. This experiment follows the same setup as the one of \cref{fig:mainexpe}. 
In particular, \cref{fig:cosine1} shows this histogram for $n=500$, where interpolation of the training data happens (see \cref{fig:train}); and \cref{fig:cosine2} shows this histogram for $n=5\ 000$, where interpolation of the training data does not happen anymore, but the network generalizes well to unseen data.

While a majority of the neurons is already nicely aligned with the true OLS estimator in the $n=500$ case, an important fraction of them are not aligned with this estimator (69\% of them have a cosine similarity smaller than $0.9$ in absolute value). These unaligned neurons contribute to a prediction function that significantly differs from the OLS one.
On the other hand, nearly all neurons are aligned with this true estimator as $n$ grows larger (91\% of them have a cosine similarity larger than $0.9$ in absolute value), confirming the predictions of \cref{thm:noconvergence}. As explained above, there are still a few vectors that are disaligned with the OLS estimator here, but they are only a small fraction and thus have almost no impact on the estimated function.

\section{Conclusion}
This work illustrates on a simple linear example the phenomenon of non-convergence of the parameters towards a global minimum of the training loss, despite overparametrization. This non-convergence actually yields a simplicity bias on the final estimator, which can lead to an optimal fit of the true data distribution. A similar phenomenon has been observed on more complex and realistic settings \citep{yoon2023diffusion,kadkhodaie2023generalization,raventos2024pretraining}.
 However, a theoretical analysis remains out of reach in these cases. It is still unclear whether the observed non-convergence arises from the early alignment mechanism proposed in our work, from stability issues as suggested by \citet{qiao2024stable}, from other factors, or from a combination of these effects.

Our result is proven via the description of the early alignment phase. Besides the specific data example considered in \cref{sec:noconvergence}, we also provide concentration bounds on the extremal vectors driving this early alignment. We believe these bounds (\cref{thm:Dconcentration}) can be used in subsequent works to better understand this early phase of the training dynamics, and how it yields biases towards simple estimators.

\section*{Acknowledgements}

This work was supported by the Swiss National Science Foundation (grant number 212111) and by an unrestricted gift from Google.

\section*{Impact Statement}

This paper presents work whose goal is to advance the field of 
Machine Learning. There are many potential societal consequences 
of our work, none which we feel must be specifically highlighted here.

\bibliographystyle{plainnat}
\bibliography{main.bib}

\newpage
\appendix
\onecolumn
\addcontentsline{toc}{section}{Appendix} 
\part{Appendix} 
\parttoc 

\section{Additional experiments}\label{app:expe}

\subsection{Experimental details}\label{app:expedetails}

In the experiments of \cref{fig:mainexpe}, we initialized two-layer ReLU networks (without bias term) with $m=10\ 000$ neurons, initialized i.i.d. for each component as a Gaussian of variance $\frac{10^{-5}}{\sqrt{m}}$. We then generated training samples as 
\begin{equation*}
y_k = \beta^{\star \,\top} x_k + \eta_k,
\end{equation*}
where $\eta_k$ are drawn i.i.d. as centered Gaussian of variance $\sigma^2=0.09$, $x_k$ are drawn i.i.d. as centered Gaussian variables and $\betastar$ is fixed, without loss of generality, to $\betastar=(1,0,\ldots,0)$. The dimension is fixed to $d=5$. We then train these networks on training datasets of different sizes (each dataset is resampled from scratch).

The neural networks are trained via stochastic gradient descent (SGD), with batch size $32$ and learning rate $0.01$. To ensure that we reached convergence of the parameters, we train the networks for $8\times 10^6$ iterations of SGD, where the training seems stabilized.

All the experiments were run on a personal MacBook Pro, for a total compute time of approximately 100 hours. The code can be found at \url{github.com/eboursier/simplicity_bias}.

\subsection{GeLU activation} \label{app:gelu}

Our theoretical results can be directly extended to any homogeneous activation function, i.e., leaky ReLU activation. Yet, the theory draws different conclusions for differentiable activations functions and claims that for infinitely wide neural networks, the parameters should interpolate the data at convergence \citep{chizat2018global}. This result yet only holds for infinitely wide networks, and it remains unknown how wide a network should be to actually reach such an interpolation in practice. \cref{fig:gelu} below presents experiments similar to \cref{sec:expe}, replacing the ReLU activation by the differentiable GeLU activation \citep{hendrycks2016gaussian}. This activation is standard in modern large language models. Notably, it is used in the GPT2 architecture, which was used in the experiments of \citet{raventos2024pretraining}.

\begin{figure}[htbp] 
\centering
\subfigure[Evolution of train loss.]{\includegraphics[width=0.45\linewidth, trim=0.4cm 0.5cm 0.4cm 0.4cm, clip]{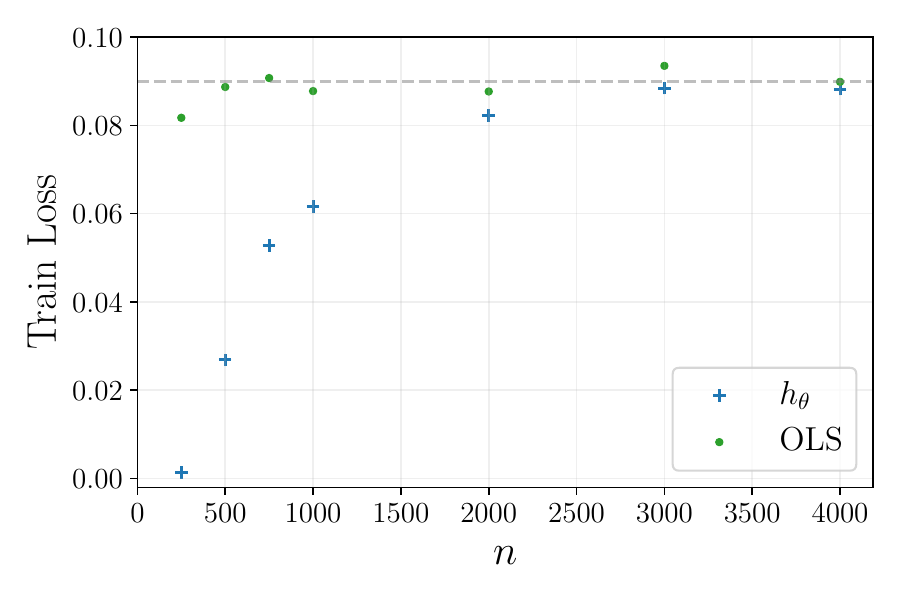}\label{fig:gelu_train}}
\hfill
\subfigure[Evolution of test loss.]{\includegraphics[width=0.45\linewidth, trim=0.4cm 0.5cm 0.4cm 0.4cm, clip]{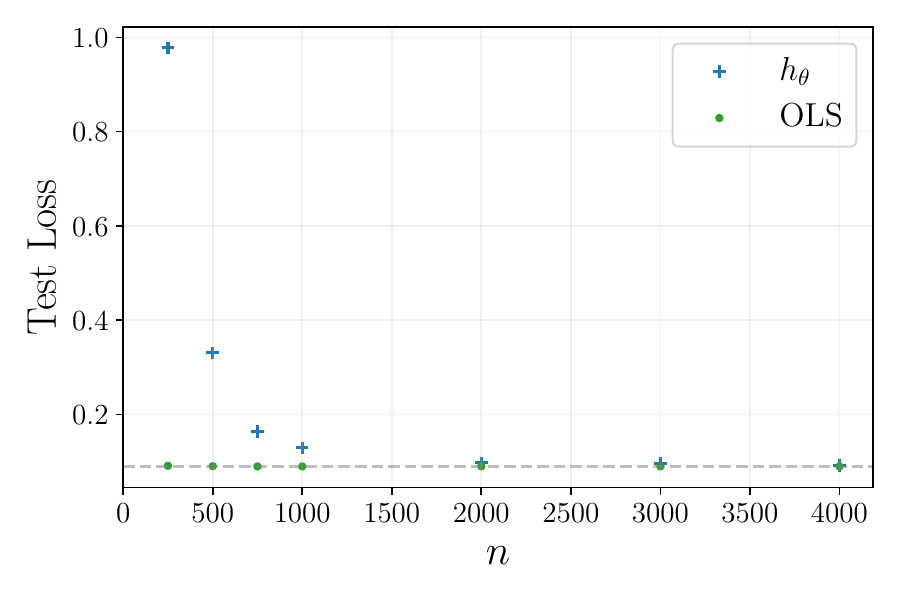}\label{fig:gelu_test}}
  \caption{\label{fig:gelu}Evolution of both train and test losses at convergence with respect to the number of training samples, with GeLU activation.}
\end{figure}

While infinitely wide GeLU networks should overfit, even very wide networks ($m=10\ 000$) are far from this behavior in practice. In particular, we observe a phenomenon similar to \cref{sec:expe} in \cref{fig:gelu}. Surprisingly, it even seems that interpolation is harder to reach with GeLU activation, as the network is already unable to interpolate for $n=500$ training samples. We believe this is due to the fact that GeLU is close to a linear function around the origin (corresponding to our small initialization regime), making it harder to overfit noisy labels.

\subsection{Momentum based optimizers}\label{app:adam}

Our theoretical results hold for Gradient Flow, which is a first order approximation of typical gradient methods such as Gradient Descent (GD) or Stochastic Gradient Descent (SGD) \citep{li2019stochastic}. Yet, recent large models implementations typically use different, momentum based algorithms, such as Adam \citep{kingma2014adam} or AdamW \citep{loshchilov2017decoupled}. To illustrate the generality of the optimization threshold we proved in a specific theoretical setting, we consider in \cref{fig:geluadam} below the same experiments as in \cref{sec:expe}, with the exception that i) we used GeLU activation functions (as in \cref{app:gelu}) and ii) we minimized the training loss through the Adam optimizer, with pytorch default hyperparameters. 

We focus on Adam rather than AdamW here to follow the experimental setup of \citet{raventos2024pretraining} and because our focus is on implicit regularization, thus avoiding explicit regularization techniques.

\begin{figure}[htbp] 
\centering
\subfigure[Evolution of train loss.]{\includegraphics[width=0.45\linewidth, trim=0.4cm 0.5cm 0.4cm 0.4cm, clip]{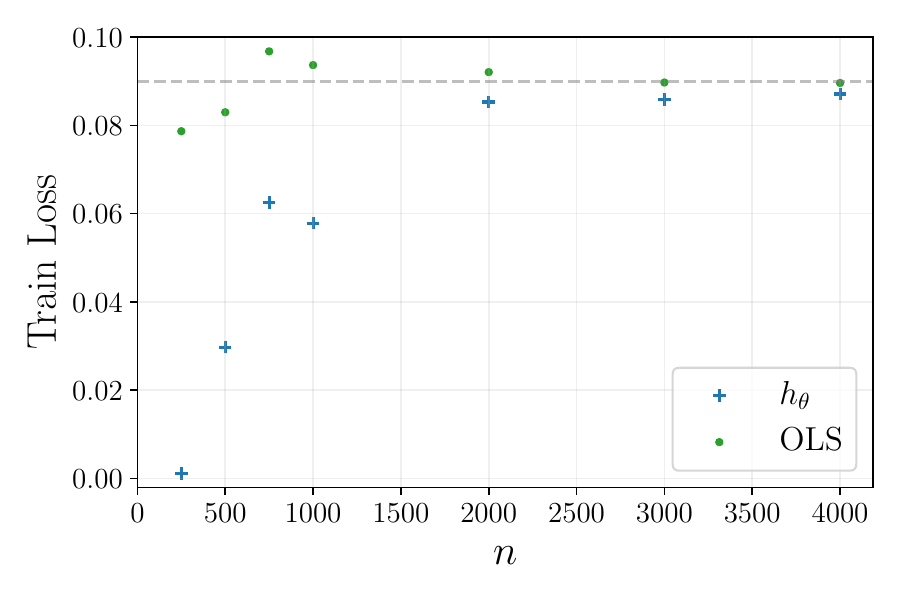}\label{fig:geluadam_train}}
\hfill
\subfigure[Evolution of test loss.]{\includegraphics[width=0.45\linewidth, trim=0.4cm 0.5cm 0.4cm 0.4cm, clip]{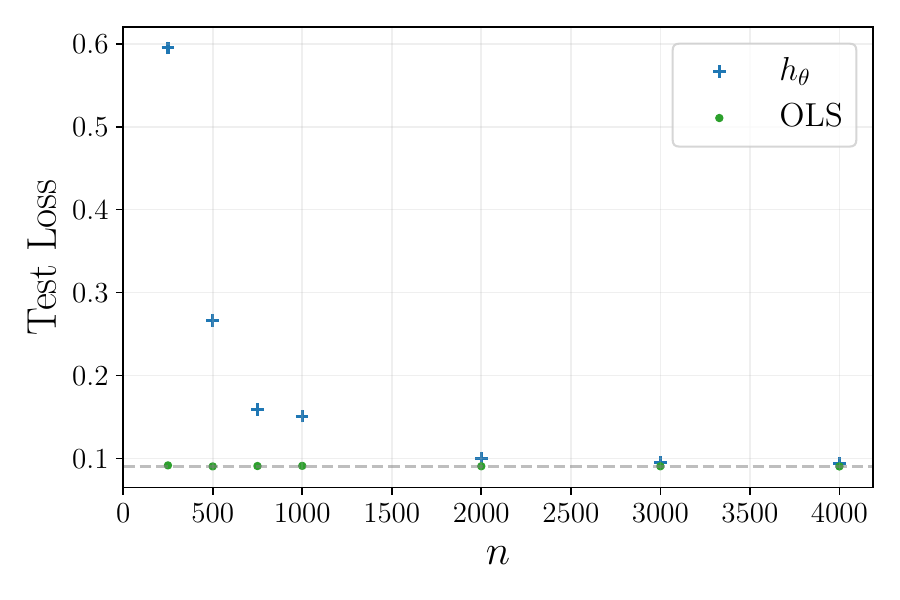}\label{fig:geluadam_test}}
  \caption{\label{fig:geluadam}Evolution of both train and test losses at convergence with respect to the number of training samples, with GeLU activation and Adam optimizer.}
\end{figure}

The observed results are very similar to the ones of \cref{fig:gelu}, leading to similar conclusions than \cref{app:gelu} and the fact that considering Adam rather than SGD does not significantly change the final results.

\subsection{Stability of minima}\label{app:stable}

\begin{figure}[htbp]
\centering
\includegraphics[width=0.5\linewidth,  trim=0.4cm 0.5cm 0.4cm 0.4cm, clip]{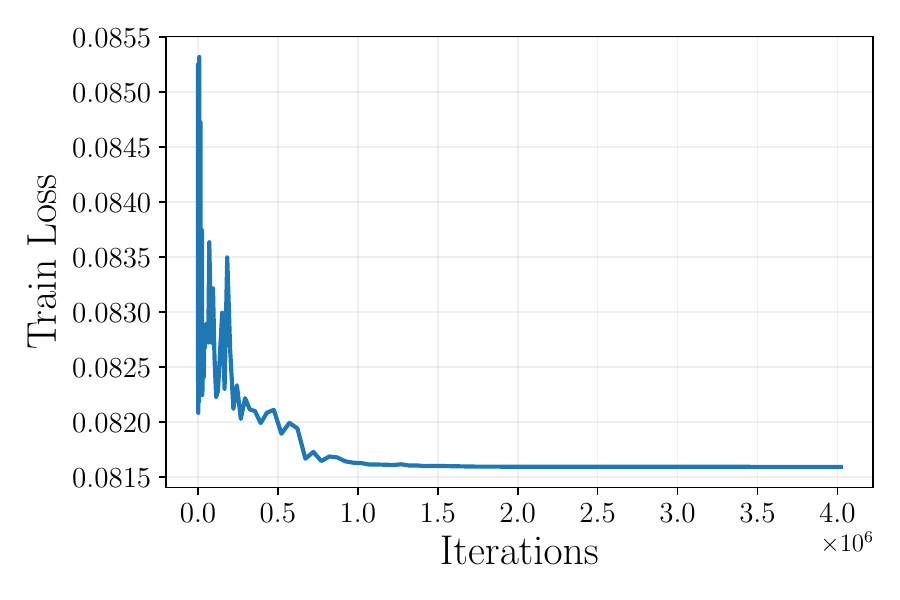}
\caption{\label{fig:stable}Evolution of training loss from warm restart with a decaying learning rate schedule ($n=8000, d=5$).}
\end{figure}

\citet{qiao2024stable} argue that the non-convergence of the estimator towards interpolation is due to the instability of global minima. More precisely they claim that for large stepsizes, gradient descent (GD) cannot stabilize around global minima of the loss for large values of $n$. We present an additional experiment in this section, illustrating that this non-convergence is not due to an instability of the convergence point of (S)GD, but to it being a stationary point of the loss as predicted by our theory.

For that, we consider a neural network initialized from the final point (warm restart) of training for $8\ 000$ samples in the experiment of \cref{fig:mainexpe}.\footnote{Another relevant experiment is to train from scratch (no warm restart) with a smaller learning rate. When running the experiment of \cref{app:expedetails} with a smaller learning rate $0.001$, we observe again that the parameters at convergence correspond to the OLS estimator.} We then continue training this network on the same training dataset, with a decaying learning rate schedule. Precisely, we start with a learning rate of $0.01$ as in the main experiment, and multiply the learning rate by $0.85$ every $50\ 000$ iterations of SGD, so that after $4\times 10^6$ iterations, the final learning rate is of order $10^{-8}$.

We observe on \cref{fig:stable} that the training loss does not change much from the point reached at the end of training with the large learning rate $0.01$. Indeed, the training loss was around $0.082$ at the end of this initial training, which is slightly less than the noise level ($0.09$). While there seems to be some stabilization happening at the beginning of this decaying schedule, the training loss seems to converge to slightly more than $0.0815$, confirming that the absence of interpolation is not due to an instability reason, but rather to a convergence towards a spurious stationary point of the loss.

\subsection{Influence of dimensionality}\label{app:dim}

\cref{thm:noconvergence} predicted an optimization threshold scaling in $\bigO{d^3\log d}$. However, the experiments of \cref{sec:expe} consider a fixed dimension ($d=5$), making it unclear how tight is this theoretical optimization threshold and whether a similar dependency in the dimension is observed in practice. To investigate further this dependency in the dimension, we present in this section experiments in the same setup described in \cref{app:expedetails}, with the sole exception that the dimension is larger, fixed to $d=10$.

\cref{fig:d10} illustrates the evolution of both the train and test losses as the number of training samples increases in this larger dimension setting. In that case, the optimization threshold seems much larger: interpolation stops happening around $n=10\ 000$ samples, and an estimation close to the OLS estimator really starts happening at much larger values of $n$, around $n=80\ 000$. 

\begin{figure}[htbp] 
\centering
\subfigure[Evolution of train loss.]{\includegraphics[width=0.45\linewidth, trim=0.4cm 0.5cm 0.4cm 0.4cm, clip]{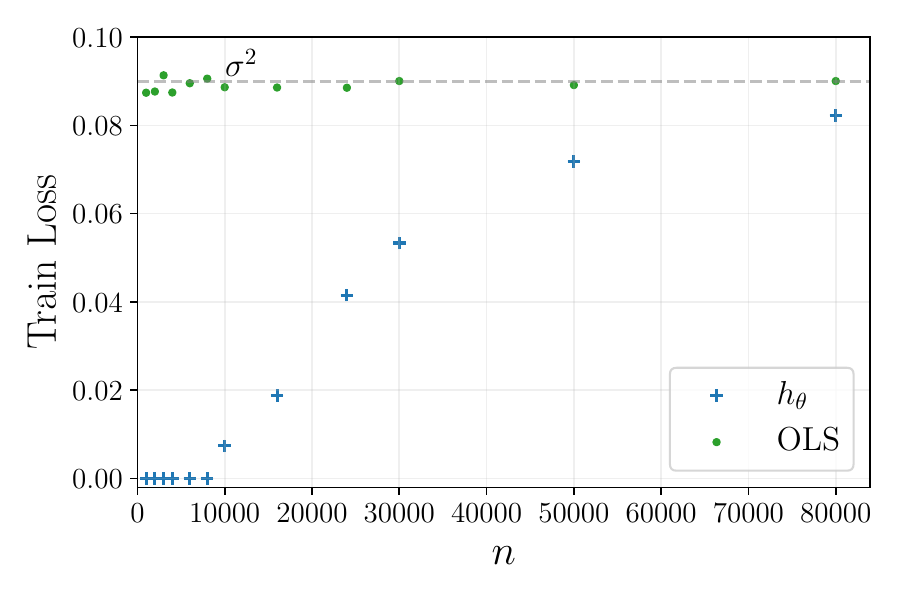}\label{fig:d10_train}}
\hfill
\subfigure[Evolution of test loss.]{\includegraphics[width=0.45\linewidth, trim=0.4cm 0.5cm 0.4cm 0.4cm, clip]{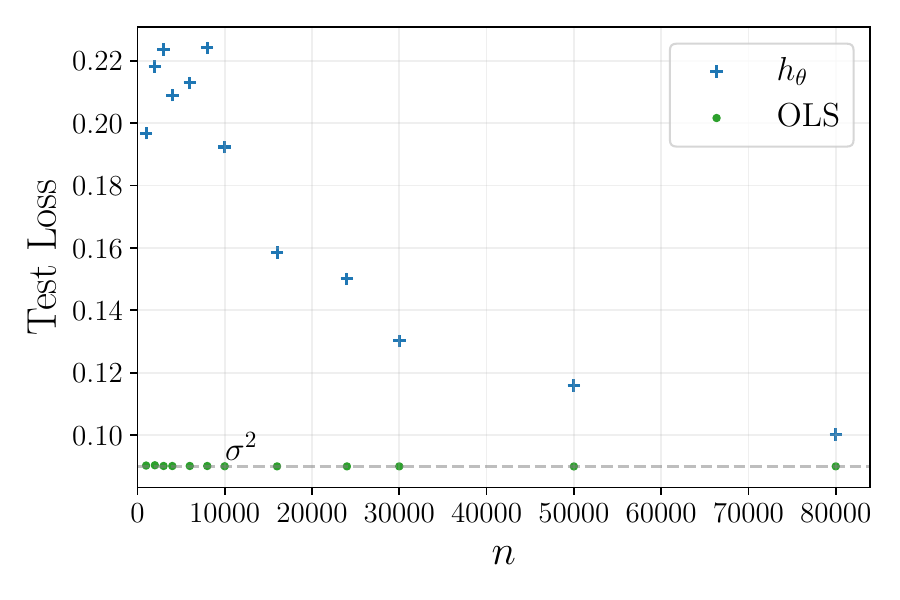}\label{fig:d10_test}}
  \caption{\label{fig:d10}Evolution of both train and test losses at convergence with respect to the number of training samples, with dimension $d=10$.}
\end{figure}

Comparing with the $d=5$ case, it thus seems that the point at which interpolation stops indeed seems to roughly scale in $d^3$. However, this scaling seems even larger for the point where the estimator corresponds to the OLS one. 
We believe that this discrepancy is due to the differences between our theoretical and experimental setups, and in particular to the fact that multiple intermediate neurons can grow in our experimental setup (see \textit{Limitations and generality} paragraph in \cref{sec:discussion}).

\subsection{5 ReLU teacher network}\label{app:relus}

This section presents an additional experiment with a more complex data model. More precisely, we consider the exact same setup than \cref{sec:expe} (described in \cref{app:expedetails}), with the difference that the labels $y_k$ are given by
\begin{equation*}
y_k = f^{\star}(x_k) + \eta_k,
\end{equation*}
where $f^{\star}$ is a $5$ ReLU network:
\begin{equation*}
f^{\star}(x_k) = \frac{1}{5}\sum_{i=1}^5 (x_k^{\top}\betastar_i)_+.
\end{equation*}
The parameters $\betastar_i$ are drawn i.i.d. at random following a standard Gaussian distribution. We use the exact same $\betastar_i$ across all the runs for different values of $n$. Also, $x_k$ and $\eta_k$ are generated in the same way as described in  in \cref{app:expedetails}.

\cref{fig:reluexpe} also presents the evolution of the train and test losses as the number of training samples varies. We observe a behavior similar to \cref{fig:mainexpe}, where interpolation is reached for small values of $n$, and is not reached anymore after some threshold $\nstar$. While the test loss is far from the optimal noise variance before this threshold, it then becomes close to it afterwards. 

Yet, this transition from interpolation to generalization is slower in the $5$ ReLU teacher case than in the linear one. Indeed, while interpolation does not happen anymore around $n=2000$ in both cases, much more samples (around $\nstar=17000$) are needed to have a simultaneously a training and testing loss close to the noise variance. 
These experiments suggest that the behavior predicted by \cref{thm:noconvergence} for a linear model also applies in more complex models such as the $5$ ReLU teacher, but that the transition from interpolation to generalization can happen more slowly or with more training samples depending on the setting. 

\begin{figure}[htbp] 
\centering
\subfigure[Evolution of train loss.]{\includegraphics[width=0.45\linewidth, trim=0.4cm 0.5cm 0.4cm 0.4cm, clip]{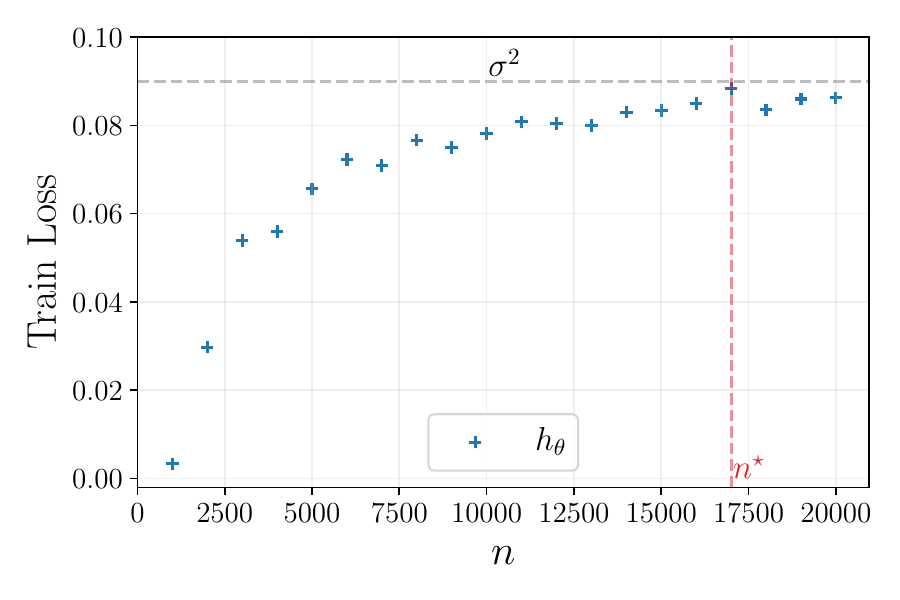}\label{fig:trainrelu}}
\hfill
\subfigure[Evolution of test loss.]{\includegraphics[width=0.45\linewidth, trim=0.4cm 0.5cm 0.4cm 0.4cm, clip]{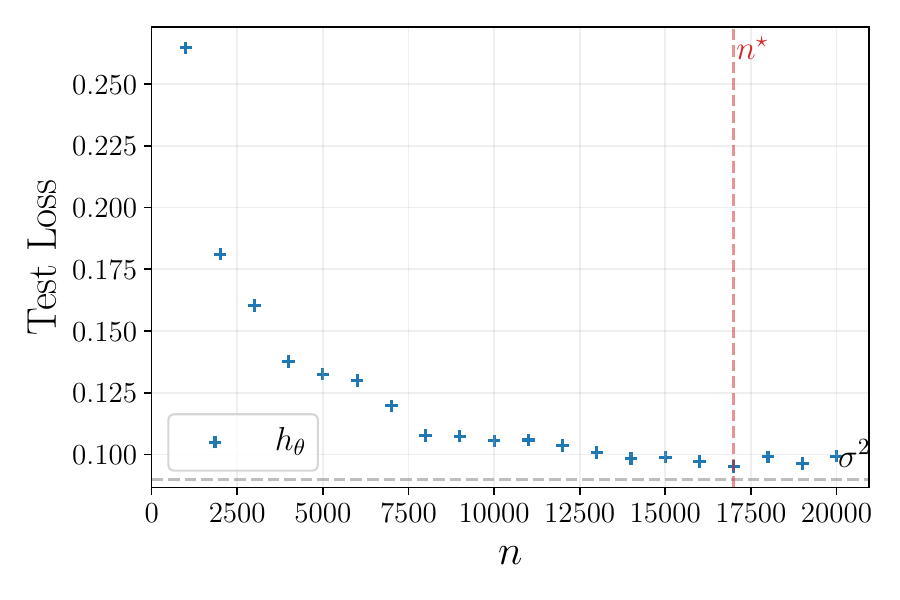}\label{fig:testrelu}}
  \caption{\label{fig:reluexpe}Evolution of both train and test losses at convergence with respect to the number of training samples with a $5$ ReLU teacher. $\sigma^2$ corresponds to the noise variance $\E[\eta^2]$.}
\end{figure}

The slight difference with \cref{fig:mainexpe} is that this optimization threshold here seems to appear for larger values of $n$. 

\section{Additional Discussions}

\paragraph{Double descent.}
Double descent originally refers to the fact that the test loss obtained at convergence does not behave monotonically with the number of model parameters. Recently, different types of double descent have been proposed \citep{nakkiran2021deep}. Notably, \citet{henighan2023superposition} study a data double descent, where the test loss follows a ``double descent'' shape when plotted against the number of training examples. The phenomenon we highlight here is different, as our test loss monotonically decreases with respect to the number of training points, as illustrated in both our experiments and \cref{fig:regimes}. 

However, the toy experiments of \citet{henighan2023superposition}  illustrate a similar phenomenon: for a sufficiently large number of training points, the training loss remains high while the model learns optimal features. It remains unclear though whether this high training loss stems from an underparametrized regime (i.e., the model lacks sufficient capacity to memorize the data) or if optimization fails to reach the empirical risk minimizer in their setup.

\paragraph{Feature learning and NTK regimes.} We distinguish in this work between feature learning and the NTK/lazy regime, as they involve fundamentally different training dynamics \citep[see][for an in-depth discussion]{chizat2019lazy}. Our work specifically focuses on the feature learning regime with small initialization, as indicated by our initialization choice (\cref{eq:init1}), where both the inner and outer layers scale as $\frac{1}{\sqrt{m}}$.

In contrast, in the NTK/lazy regime (corresponding to large initialization scales), theory predicts that interpolation should occur at convergence, which is contrary to our main result. However, empirically demonstrating this interpolation in our toy model (with large $n$) is computationally challenging, as it would require an extremely large number of parameters.

\section{Proof of \cref{thm:Dconcentration}}\label{sec:Dconcproof}

We recall \cref{thm:Dconcentration} below.

\Dconc*

\begin{proof}
We first show a similar result on the following expectation 
\begin{equation}\label{eq:tildeD}
\E_{\bX,\by}\left[\sup_{w\in\bS_{d-1}}\|D_n(w) - D(w)\|_2\right] = \bigO{\sqrt{\frac{d\log n}{n}\E[\|yx\|^2_2]}},
\end{equation}
where we recall $D_n(w) = \frac{1}{n}\sum_{k=1}^n \iind{w^{\top}x_k>0}y_k x_k$. We bound this expectation using typical uniform bound techniques for empirical processes.

A symmetrization argument allows to show, for i.i.d. Rademacher random variables $\varepsilon_k \in\{-1,1\}$ \citep[see][Lemma 2.3.1.]{van1997weak}:
\begin{align}\label{eq:symm1}
\E_{\bX,\by}\left[\sup_{w\in\bS_{d-1}}\|D_n(w) - D(w)\|_2\right] & \leq 2\ \E_{\bX,\by}\left[\E_{\pmb{\varepsilon}}\left[\sup_{w\in\bS_{d-1}}\left\|\frac{1}{n}\sum_{k=1}^n \iind{w^{\top}x_k>0} \varepsilon_k y_k x_k \right\|_2\bigm| \bX,\by\right] \right].
\end{align}

From there, it remains to bound for any value of $\bX,\by$ the conditioned expectation $\E_{\pmb{\varepsilon}}[\cdot \mid \bX,\by]$. We consider in the following a fixed value of $\bX,\by$. Note that the vector $\sum_{k=1}^n \iind{w^{\top}x_k>0} \varepsilon_k y_k x_k$, actually only depends on $w$ in the value of the vector $(\iind{w^{\top}x_k>0})_{k\in[n]}$. Define
\begin{equation}\label{eq:supeq}
\cA(\bX,\by) = \left\{ (\iind{w^{\top}x_k>0})_{k\in[n]} \mid w\in\R^d\right\}.
\end{equation}
We thus have the equality:
\begin{equation*}
\sup_{w\in\bS_{d-1}}\left\|\frac{1}{n}\sum_{k=1}^n \iind{w^{\top}x_k>0} \varepsilon_k y_k x_k \right\|_2 = \sup_{\bf{u}\in\cA(\bX,\by)}\left\|\frac{1}{n}\sum_{k=1}^n u_k \varepsilon_k y_k x_k \right\|_2.
\end{equation*}
Moreover, classical geometric arguments \citep[see e.g.][Theorem 1]{cover1965geometrical} allow to bound $\card{\cA(\bX,\by)}$ for any $\bX,\by$:
\begin{align}
\card{\cA(\bX,\by)} & \leq 2 \sum_{k=0}^{d-1} \binom{n-1}{k}\notag\\
& =\bigO{n^d}. \label{eq:nsectors}
\end{align}

From there, we will bound individually for each $\bf{u}\in\cA(\bX,\by)$ the norm of $\frac{1}{n}\sum_{k=1}^n u_k \varepsilon_k y_k x_k$ and use a union bound argument.

Let $\bf{u}\in\cA(\bX,\by)$. Define $Z\in\R^{d\times n}$ the matrix whose column $k$ is given by $Z^{(k)} = \frac{1}{n}u_k y_k x_k$. Then note that $\frac{1}{n}\sum_{k=1}^n u_k \varepsilon_k y_k x_k = Z \pmb{\varepsilon}$. Hanson-Wright inequality then allows to bound the following probability \citep[see][Theorem 2.1]{rudelson2013hanson} for some universal constant $c>0$ and any $t\geq 0$:
\begin{align*}
\bP_{\pmb{\varepsilon}}\left( \big| \|Z\pmb{\varepsilon}\|_2 - \|Z\|_{\Fro}   \big| > t \bigm| \bX,\by\right) & \leq 2 e^{-\frac{ct^2}{\|Z\|_{\op}^2}},
\end{align*}
where $\|Z\|_{\Fro}$ and $\|Z\|_{\op}$ respectively denote the Frobenius and operator norm of $Z$. In particular, noting that $\|Z\|_{\op}\leq \|Z\|_{\Fro}$, this last equation implies that for any $t>0$
\begin{equation}\label{eq:HW1}
\bP_{\pmb{\varepsilon}}\left( \|Z\pmb{\varepsilon}\|_2 > (1+t)\|Z\|_{\Fro}  \bigm| \bX,\by\right)  \leq 2 e^{-c t^2}.
\end{equation}
Moreover, note that
\begin{align*}
\|Z\|_{\Fro} & = \sqrt{\sum_{k=1}^n \|Z^{(k)}\|_2^2}\\
& = \sqrt{\sum_{k=1}^n \frac{1}{n^2}\| u_k y_k x_k\|_2^2}
\leq \sqrt{\frac{1}{n}C(Z)},
\end{align*}
where $C(Z) = \frac{1}{n}\sum_{k=1}^n \|y_k x_k\|_2^2$ does not depend on $\bf{u}$.

Rewriting \cref{eq:HW1} with this last inequality, and with $\delta = 2e^{-ct^2}$, we finally have for each $\bf{u}\in\cA(\bX,\by)$:
\begin{equation*}
\bP_{\pmb{\varepsilon}}\left( \|Z\pmb{\varepsilon}\|_2 > (1+\sqrt{\frac{1}{c}\ln(2/\delta)})\sqrt{\frac{C(Z)}{n}}  \bigm| \bX,\by\right)  \leq \delta.
\end{equation*}
Considering a union bound over all the $\bf{u}\in\cA(\bX,\by)$, we have for some universal constant $c'>0$, thanks to \cref{eq:nsectors}:
\begin{equation}\label{eq:HWunion}
\bP_{\pmb{\varepsilon}}\left( \exists \mathbf{u}\in \cA(\bX,\by), \|Z\pmb{\varepsilon}\|_2 > \left(1+\sqrt{c'(\log(2/\delta)+d\log(n)+1)}\right)\sqrt{\frac{C(Z)}{n}}  \bigm| \bX,\by\right) \leq \delta.
\end{equation}
Moreover, conditioned on $\bX,\by$,  $\|Z\pmb{\varepsilon}\|_2$ is almost surely bounded by $\sqrt{n}\|Z\|_{\op}$, and so by $\sqrt{C(Z)}$. A direct bound on the expectation can then be derived using \cref{eq:HWunion} with $\delta = n^{-d}$:
\begin{align*}
\E_{\pmb{\varepsilon}}\left[ \sup_{\bf{u}\in\cA(\bX,\by)}\left\|\frac{1}{n}\sum_{k=1}^n u_k \varepsilon_k y_k x_k \right\|_2\right] & =\bigO{\sqrt{\frac{d \log n}{n}+n^{-d}}}\sqrt{C(Z)}. 
\end{align*}
Wrapping up with \cref{eq:supeq} and \cref{eq:symm1} then allows to derive \cref{eq:tildeD},
\begin{align*}
\E_{\bX,\by}\left[\sup_{w\in\bS_{d-1}}\|D_n(w) - D(w)\|_2\right] & = \bigO{\sqrt{\frac{d \log n}{n} }}\E_{\bX,\by}\left[\sqrt{C(Z)}\right]\\
& \leq \bigO{\sqrt{\frac{d \log n}{n} }}\sqrt{\E_{\bX,\by}\left[C(Z)\right]}\\
& = \bigO{\sqrt{\frac{d \log n}{n} }}\sqrt{\E\left[\|yx\|_2^2\right]}.
\end{align*}
\cref{lemma:supeqsectors} below then allows to conclude.
\begin{restatable}{lem}{Supeqsectors}\label{lemma:supeqsectors}
If the marginal law of $x$ is continuous with respect to the Lebesgue measure, then almost surely:
\begin{equation*}
\sup_{w\in\bS_{d-1}}\sup_{D_n\in\D_n(w)}\|D_n - D(w)\|_2 \leq \sup_{w\in\bS_{d-1}}\|D_n(w) - D(w)\|_2,
\end{equation*}
where $D_n(w) = \frac{1}{n}\sum_{k=1}^n \iind{w^{\top}x_k>0}y_k x_k$.
\end{restatable}
\end{proof}
\subsection{Proof of \cref{lemma:supeqsectors}}
First observe that if the marginal law of $x$ is continuous, then $D$ is continuous with respect to $w$.

Consider any $w\in\bS_{d-1}$. We recall that the set $\D_n(w)$ is defined as
\begin{equation*}
\D_n(w) = \Big\{ -\frac{1}{n}\sum_{k=1}^n \eta_k y_k x_k \ \Big|\ \forall k\in[n], \eta_k \begin{cases} \in [0,1] \text{ if }\langle w_j^t,x_k\rangle =0 \\ =1 \text{  if }\langle w_j^t,x_k\rangle>0 \\
=0 \text{ otherwise}
\end{cases}\Big\}.
\end{equation*}
If all the values $w^{\top}x_k$ are non-zero, then $\D_n(w)$ is the singleton given by $D_n(w)$ and thus 
\begin{equation*}
\sup_{D_n\in\D_n(w)}\|D_n - D(w)\|_2 = \|D_n(w) - D(w)\|_2.
\end{equation*}

Otherwise, if $w^{\top}x_k=0$ for at least one $k$, observe that\footnote{This observation directly follows from the definition of the Clarke subdifferential.}
\begin{equation*}
\D_n(w) = \liminf_{\substack{\varepsilon\to 0\\\varepsilon>0}}\convex(\{D_n(w') \mid w' \in \cS \text{ and } \|w-w'\|_2\leq \varepsilon\}),
\end{equation*}
where
\begin{equation*}
\cS = \{w' \in \bS_{d-1} \mid w'^{\top}x_k\neq 0 \text{ for all }k\}.
\end{equation*}
In other words, for any $D_n\in\D_n(w)$, $w$ can be approached arbitrarily closed by vectors $w_i\in\cS$ such that for some convex combination $\pmb\eta$,
\begin{equation*}
D_n = \sum_{i}\eta_i D_n(w_i).
\end{equation*}
From then, it comes that
\begin{align*}
\|D_n - D(w)\| & \leq \sum_i \eta_i \|D_n(w_i)-D(w)\|\\
&\leq \sum_i \eta_i (\|D_n(w_i)-D(w_i)\|+\|D(w_i)-D(w)\|)\\
& \leq \sup_{w'\in\cS}\|D_n(w')-D(w')\| + \sum_i \eta_i \|D(w_i)-D(w)\|.
\end{align*}
Since $D$ is continuous and the $w_i$ can be chosen arbitrarily close to $w$, the right sum can be chosen arbitrarily close to $0$.

In particular, we have shown that for any $D_n\in\D_n(w)$,
\begin{equation*}
\|D_n - D(w)\| \leq \sup_{w'\in\cS}\|D_n(w')-D(w')\|.
\end{equation*}
This concludes the proof of \cref{lemma:supeqsectors}.\qed

\section{Probability tail bound version of \cref{thm:Dconcentration}}\label{app:tailbound}

While \cref{thm:Dconcentration} bounds the maximal deviation of $D_n - D(w)$ in expectation, a high probability tail bound is also possible, as given by \cref{thm:Dtailbound} below.

\begin{thm}\label{thm:Dtailbound}
If the marginal law of $x$ is continuous with respect to the Lebesgue measure, then for any $n\in\N$ and $M\geq \E\left[\left\|yx\right\|^2\right]$,
\begin{multline*}
\bP_{\bX,\by}\left[\sup_{w\in\bS_{d-1}}\sup_{D_n\in\D_n(w)}\|D_n - D(w)\|_2>4\left(1+\sqrt{c'(\log(2/\delta)+d\log(n)+1)}\right)\sqrt{\frac{M}{n}}\right]\\ 
\leq\frac{4}{3}\delta + \frac{4}{3}\bP_{\bX,\by}\left[\frac{1}{n}\sum_{k=1}^n \|y_k x_k\|^2 > M \right].
\end{multline*}
\end{thm}

\begin{proof}
The proof follows the same lines as the proof of \cref{thm:Dconcentration} in \cref{sec:Dconcproof}. In particular, we first want to bound in probability the term $\sup_{w\in\bS_{d-1}}\|D_n(w) - D(w)\|_2$. To this end, a probabilistic symmetrization argument \citep[][Lemma 2.3.7.]{van1997weak} yields
for any $t>0$
\begin{equation}\label{eq:probsymm}
\beta_n(y) \bP_{\bX,\by}\left[\sup_{w\in\bS_{d-1}}\|D_n(w) - D(w)\|_2>t\right]\leq 2\bP_{\bX,\by}\left[ \bP_{\pmb{\varepsilon}}\left[ \sup_{w\in\bS_{d-1}} \left\|\frac{1}{n} \sum_{k=1}^n \varepsilon_k \iind{w^{\top}x_k>0} y_k x_k\right\|  > \frac{t}{4} \bigm| \bX,\by\right] \right],
\end{equation}
where $\beta_n(t)=1-\frac{4n}{t^2}\sup_{w,w'\in\bS_{d-1}} \var(\iind{w^{\top}x>0}yw'^{\top}x)$. In particular here, $\beta_n(t)\geq 1 -\frac{4n}{t^2}\E\left[\left\|yx\right\|^2\right]$. Moreover, we already showed \cref{eq:HWunion} in the proof of \cref{thm:Dconcentration}, which states
\begin{equation*}\label{eq:HWunion2}
\bP_{\pmb{\varepsilon}}\left( \sup_{w\in\bS_{d-1}}\left\|\frac{1}{n}\sum_{k=1}^n \varepsilon_k \iind{w^{\top}x_k>0} y_k x_k\right\| > \left(1+\sqrt{c'(\log(2/\delta)+d\log(n)+1)}\right)\sqrt{\frac{C(\bX,\by)}{n}}  \bigm| \bX,\by\right) \leq \delta,
\end{equation*}
where $C(\bX,\by)=\frac{1}{n}\sum_{k=1}^n \|y_k x_k\|^2$.

\cref{eq:probsymm} then rewrites for any $M>0$:
\begin{multline*}
\bP_{\bX,\by}\left[\sup_{w\in\bS_{d-1}}\|D_n(w) - D(w)\|_2>4\left(1+\sqrt{c'(\log(2/\delta)+d\log(n)+1)}\right)\sqrt{\frac{M}{n}}\right]\\ \leq \beta_n(t)^{-1} \left(\delta + \bP_{\bX,\by}\left[\frac{1}{n}\sum_{k=1}^n \|y_k x_k\|^2 > M \right]\right),
\end{multline*}
with
\begin{gather*}
t = 4\left(1+\sqrt{c'(\log(2/\delta)+d\log(n)+1)}\right)\sqrt{\frac{M}{n}}.
\end{gather*}
Note that for any $M\geq \E\left[\left\|yx\right\|^2\right]$, $\beta_n(t)\geq \frac{3}{4}$, which implies 
\begin{multline*}
\bP_{\bX,\by}\left[\sup_{w\in\bS_{d-1}}\|D_n(w) - D(w)\|_2>4\left(1+\sqrt{c'(\log(2/\delta)+d\log(n)+1)}\right)\sqrt{\frac{M}{n}}\right]\\ 
\leq\frac{4}{3}\delta + \frac{4}{3}\bP_{\bX,\by}\left[\frac{1}{n}\sum_{k=1}^n \|y_k x_k\|^2 > M \right].
\end{multline*}
\cref{thm:Dtailbound} then follows, thanks to \cref{lemma:supeqsectors}.
\end{proof}
\cref{coro:Dtailbound} below provides a simpler tail bound, directly applying \cref{lemma:supeqsectors} with Chebyshev's inequality to bound $\bP_{\bX,\by}\left[\frac{1}{n}\sum_{k=1}^n \|y_k x_k\|^2 > M \right]$. Stronger tail bounds can be provided with specific conditions on the random variables $x_k$ and $y_k$, but the one of \cref{coro:Dtailbound} is enough for our use in \cref{sec:noconvergence}.

\begin{coro}\label{coro:Dtailbound}
Assume the marginal law of $x$ is continuous with respect to the Lebesgue measure. Moreover, assume $\|xy\|$ admits a fourth moment. Then
\begin{multline*}
\bP_{\bX,\by}\left[\sup_{w\in\bS_{d-1}}\sup_{D_n\in\D_n(w)}\|D_n - D(w)\|_2>4\left(1+\sqrt{c'(\log(2/\delta)+d\log(n)+1)}\right)\sqrt{\frac{2\E[\|y x\|^2]}{n}}\right]\\ 
\leq\frac{4}{3}\delta + \frac{4}{3} \frac{\E[\|y x\|^4]}{n\E[\|y x\|^2]^2}.
\end{multline*}
\end{coro}

\begin{proof}
This is a direct consequence of \cref{thm:Dtailbound}, using Chebyshev's inequality to bound $\bP_{\bX,\by}\left[\frac{1}{n}\sum_{k=1}^n \|y_k x_k\|^2 > M \right]$.
\end{proof}

\section{Proof of \cref{prop:Dconcentration}}\label{sec:proofDconcentration}

In the following proof, we define the following subsets of the unit sphere in dimension $d$ for any $\delta>0$:
\begin{gather*}
\cH = \{w\in\bS_{d-1} \mid D(w) \text{ satisfies \cref{eq:extremeinf}}\},\\
\cH(\delta) = D^{-1}\left(\bigcup_{w\in\cH}B(D(w),\delta)\right)\cap \bS_{d-1},\\
\Delta(\delta) = \min(1,\inf_{w\in\bS_{d-1}\setminus \cH(\delta)} \min\left(\|\frac{D(w)}{\|D(w)\|}-w\|, \|\frac{D(w)}{\|D(w)\|}+w\|, \|D(w)\|\right)).
\end{gather*}
Here, $B(D(w),\delta)$ denotes the open ball of radius $\delta$, centered in $D(w)$.

\begin{proof}
Since the marginal distribution of $X$ is continuous, the function $D:w\mapsto D(w)$ is continuous. In particular for any $\delta>0$, the infimum defining $\Delta(\delta)$ is reached, so that $\Delta(\delta)>0$ by definition of~$\cH$. In the following, we let $\delta=\frac{\varepsilon}{2}$. Thanks to \cref{coro:Dtailbound}, with probability at least $1-\bigO[\mu]{\frac{1}{n}}$, 
\begin{equation*}
\sup_{w\in\bS_{d-1}}\sup_{D_n\in\D_n(w)}\|D_n - D(w)\|_2 = \bigO[\mu]{\sqrt{\frac{d\log n}{n}}}.
\end{equation*}
In particular, we can choose $\nstar(\varepsilon) = \bigO[\mu]{ \frac{d\log\left(\frac{d}{\min(\Delta(\frac{\varepsilon}{2})^4,\varepsilon^2)}\right)}\min(\Delta(\frac{\varepsilon}{2})^2,\varepsilon)}$ large enough so that for any $n\geq\nstar(\varepsilon)$, with probability at least $1-\bigO[\mu]{\frac{1}{n}}$:
\begin{equation}\label{eq:concDnprop}
\sup_{w\in\bS_{d-1}}\sup_{D_n\in\D_n(w)}\|D_n - D(w)\|_2 \leq \frac{1}{2}\min\left(\Delta(\frac{\varepsilon}{2})^2, \varepsilon\right).
\end{equation}
We assume in the following of the proof that \cref{eq:concDnprop} holds.

Consider an extremal vector $D_n$ of the finite data $(\bX,\by)$. By definition, there is some $w\in\bS_{d-1}$ such that $D\in\cD_n(w)$ and either
\begin{enumerate}
\item $D_n = \mathbf{0}$,
\item or $A_n(D_n)=A_n(w)$,
\item or $A_n(D_n)=-A_n(w)$.
\end{enumerate}
In the first case, \cref{eq:concDnprop} yields that $\|D(w)\|_2<\Delta(\frac{\varepsilon}{2})$. Necessarily, by definition of $\Delta(\frac{\varepsilon}{2})$, $w\in\cH(\frac{\varepsilon}{2})$. This means by definition of $\cH(\frac{\varepsilon}{2})$ that there exists $D^{\star} \in\R^d$ satisfying \cref{eq:extremeinf}, such that
\begin{equation*}
\|D(w)-D^{\star}\|_2 \leq \frac{\varepsilon}{2}.
\end{equation*}
In particular, using \cref{eq:concDnprop} again yields $\|D_n-D^{\star}\|_2 \leq\varepsilon$.

\medskip

In the second case ($A_n(D_n)=A_n(w)$), we can assume $D_n\neq 0$. In that case, as $\frac{D_n}{\|D_n\|}$ have the same activations, $\cD_n(w)=\cD_n(\frac{D_n}{\|D_n\|})$, i.e., we can assume without loss of generality that $w=\frac{D_n}{\|D_n\|}$ here. Similarly to the first case, if $w\in\cH(\frac{\varepsilon}{2})$, then there exists $D^{\star} \in\R^d$ satisfying \cref{eq:extremeinf}, such that $\|D_n-D^{\star}\|_2 \leq\varepsilon$.

Let us show by contradiction that indeed $w\in\cH(\frac{\varepsilon}{2})$. Assume $w\not\in\cH(\frac{\varepsilon}{2})$. In particular, $\|D(w)\|_2\geq \Delta(\frac{\varepsilon}{2})$. We can now bound the norm of $\frac{D(w)}{\|D(w)\|}-w$:
\begin{align*}
\left\|\frac{D(w)}{\|D(w)\|}-w\right\|_2 &=\left\|\frac{D(w)}{\|D(w)\|}-\frac{D_n}{\|D_n\|}\right\|_2 \\
& = \left\|\frac{D(w)-D_n}{\|D(w)\|} +\frac{D_n}{\|D_n\|} \left(\frac{\|D_n\|-\|D(w)\|}{\|D(w)\|}\right) \right\|_2\\
&\leq  \frac{\|D(w)-D_n\|}{\|D(w)\|} + \frac{\big|\|D_n\|-\|D(w)\|\big|}{\|D(w)\|}\\
&\leq 2\frac{\|D(w)-D_n\|}{\|D(w)\|}\\
& \leq 2 \frac{\Delta(\frac{\varepsilon}{2})^2}{2\Delta(\frac{\varepsilon}{2})} = \Delta(\frac{\varepsilon}{2}).
\end{align*}
By definition of $\Delta(\frac{\varepsilon}{2})$, this actually implies that $w\in\cH(\frac{\varepsilon}{2})$, which contradicts the initial assumption. We thus indeed have $w\in\cH(\frac{\varepsilon}{2})$, leading to the existence of a $D^{\star}$ with the wanted properties such that $\|D_n-D^{\star}\|_2 \leq\varepsilon$. In the third case ($A_n(D_n)=-A_n(w)$), symmetric arguments lead to the same conclusion, which concludes the proof of \cref{prop:Dconcentration}.
\end{proof}

\section{Proof of \cref{thm:noconvergence}}\label{sec:noconvergenceproof}

\subsection{Notations and first classical results}\label{app:notations}

In the whole \cref{sec:noconvergenceproof}, we define $\Sigma = \E[xx^{\top}]$ and 
\begin{gather*}
\Sigma_{n,+} = \frac{1}{n}\sum_{k\in\cS_+} x_k x_k^{\top} \quad,\quad
\Sigma_{n,-} = \frac{1}{n}\sum_{k\in\cS_-} x_k x_k^{\top}
\end{gather*}
and the following set of neurons:
\begin{gather*}
\cI_+ = \{i\in[m]\mid a_i(0)\geq 0\}\quad \text{and}\quad
\cI_- =  \{i\in[m]\mid a_i(0)< 0\}.
\end{gather*}

We first start by stating the following, known balancedness lemma \citep[see, e.g.,][]{arora2019fine,boursier2022gradient}.
\begin{lem}\label{lemma:balancedness}
For any $i\in[m]$ and $t\in\R_+$, $a_i(t)^2-\|w_i(t)\|^2=a_i(0)^2-\|w_i(0)\|^2$.
\end{lem}
\cref{lemma:balancedness} can be simply proved by a direct computation of the derivative of $a_i(t)^2-\|w_i(t)\|^2$. Thanks to \cref{eq:dominatedinit}, this yields that the sign $a_i(t)$ is constant over time, and thus partitioned by the sets of neurons $\cI_+$ and $\cI_-$.

Also, note that with probability $1-\frac{1}{2^{m-1}}$, the sets $\cI_+$ and $\cI_-$ are both non empty, which is assumed to hold in the following of the section.

In this section, all the $\mathcal{O},\Theta$ and $\Omega$ notations hide constants depending on the fourth moment of $\eta$, the norm of $\betastar$ and the constant $c$ of \cref{ass:noconvergence}. Note that due to the sub-Gaussian property of $x$, its $k$-th moment can be bounded as $\E[\|x\|^k]\bigO{d^{\frac{k}{2}}}$ for any $k$.

\subsection{Phase 1: early alignment}

\begin{restatable}{lem}{early}\label{lemma:phase1}
If \cref{ass:noconvergence} holds, there exists $\lambdastar=\Theta(\frac{1}{d})$ and $\nstar=\Theta(d^3\log d)$ such that for any $\lambda\leq\lambdastar$ and $\varepsilon\in(0,\frac{1}{4})$, $n\geq \nstar$ and for $\tau=\frac{\varepsilon\ln(1/\lambda)}{\|\Sigma\betastar\|}$, with probability $1-\bigO{\frac{1}{n}}$:
\begin{enumerate}
\item output weights do not change until $\tau$:
\begin{gather*}
\forall t \leq \tau,\forall j\in[m], |a_j(0)|\lambda^{2\varepsilon} \leq |a_j(t)| \leq |a_j(0)|\lambda^{-2\varepsilon};
\end{gather*}
\item all neurons align with $\pm\Sigma\betastar$:
\begin{equation*}
\forall i\in[m], \quad  \langle \frac{w_i(\tau)}{a_i(\tau)}, \Sigma\betastar\rangle = \|\Sigma\betastar\| - \bigO{\lambda^{\varepsilon}+\sqrt{\frac{d^2\log n}{n}}}.
\end{equation*}
\end{enumerate}
\end{restatable}

\begin{proof}
We start the proof by computing $D(w)$ for any $w\in\bS_{d-1}$:
\begin{align*}
D(w) & = \E[\iind{w^{\top}x>0}yx] \\
&=  \E[\iind{w^{\top}x>0} xx^{\top}\betastar]\\
& = \frac{1}{2}\left(\E[\iind{w^{\top}x>0} xx^{\top}]+\E[\iind{w^{\top}x<0}xx^{\top}] \right)\betastar \\
& = \frac{\Sigma \betastar }{2}.
\end{align*}The second inequality comes from the independence between $x$ and $\eta$, the third one comes from the symmetry of the distribution of $x$ and the last one by continuity of this distribution.

\cref{coro:Dtailbound} additionally implies that for some $\nstar = \Theta(d^3\log d)$ and any $n\geq\nstar$, with probability at least $1-\bigO{\frac{1}{n}}$,
\begin{equation}\label{eq:concentrationlinear}
\sup_{w\in\bS_{d-1}}\sup_{D_n\in\D_n(w)}\|D_n-D(w)\|_2 \leq \bigO{\sqrt{\frac{d^2\log n}{n}}} \leq \frac{\alpha}{\sqrt{d}},
\end{equation}
with $\alpha = \frac{1}{4}\min(c- \sqrt{d}\|\Sigma\betastar - \betastar\|, \|\Sigma\betastar\|)$. Note\footnote{The additional $d$ dependence comes from the expectation of $\|yx\|^2$ in the square root. Additionally, the probability bound comes from the fact that $\frac{\E_{\mu}[\|y x\|^4]}{\E_{\mu}[\|y x\|^2]^2}=\bigO{1}$ here.} that $\alpha>0$ thanks to the fourth point of \cref{ass:noconvergence}. 
Moreover, using typical concentration inequality for sub-Gaussian vectors, we also have with probability $1-\bigO{\frac{1}{n}}$:
\begin{gather}
\label{eq:xnorms} \frac{\sum_{k=1}^n\|x_k\|^2}{n} \leq 2 \E_{\mu_X}[\|x\|^2] = \bigO{d}.
\end{gather}
We assume in the following of the proof that both \cref{eq:concentrationlinear,eq:xnorms} hold.

Since $D(w)=\frac{\Sigma\betastar}{2}$ for any $w$, we have for any $w\in\bS_{d-1}$, $D_n\in\D_n(w)$ and $k\in\cS_+$:
\begin{align*}
x_k^{\top}D_n & = x_k^{\top}(D_n-D(w)) + \frac{1}{2}x_k^{\top}(\Sigma\betastar - \betastar) + \frac{1}{2}x_k^{\top}\betastar\\
& \geq \|x_k\|\left(- \|D_n-D(w)\| - \frac{1}{2}\|\Sigma\betastar - \betastar\| + \frac{c}{2\sqrt{d}} \right)\\
& \geq \frac{\alpha}{\sqrt{d}}\|x_k\| >0.
\end{align*}
Similarly for any $k\in\cS_-$, $x_k^{\top}D_n<0$. This directly implies here that there are only two extremal vectors here:
\begin{gather}
D_n(\betastar) = \Sigma_{n,+} \betastar + \frac{1}{n}\sum_{k\in\cS_+} \eta_k x_k,\notag\\
D_n(-\betastar) = \Sigma_{n,-} \betastar + \frac{1}{n}\sum_{k\in\cS_-} \eta_k x_k.\label{eq:angleDn}
\end{gather}

We can now show, similarly to \citet{boursier2024early}, the early alignment phenomenon in the first phase.\footnote{We could directly reuse Theorem 1 from \citet{boursier2024early} here, but it would not allow us to choose an initialization scale $\lambdastar$ that does not depend on $n$.}

1. First note that \cref{eq:concentrationlinear} and the definition of $\alpha$ imply that for any $w$:
\begin{equation}\label{eq:normDn}
\|D_n(w)\|\leq\|\Sigma\betastar\|.
\end{equation}
We define $t_1 = \min\{t\geq 0 \mid \sum_{j=1}^m a_j(t)^2\geq \lambda^{2-4\varepsilon}\}$.

For any $i\in[m]$ and $t\in [0,t_1]$, \cref{eq:ODE} rewrites:
\begin{align*}
\left|\frac{\df a_i(t)}{\df t}\right| & = \left|w_i(t)^{\top} D_n^i(t)\right|\\
& \leq |a_i(t)| \left(\max_{w\in\bS_{d-1}}\|D_n(w)\| + \frac{\sum_{k=1}^n\|x_k\|^2}{n}\lambda^{2-4\varepsilon}\right) \\
& \leq |a_i(t)| \left(\|\Sigma\betastar\|+2\E[\|x\|^2]\lambda^{2-4\varepsilon}\right).
\end{align*}

As a consequence, a simple Gr\"onwall argument yields that for any $t\in[0,t_1]$:
\begin{align*}
|a_i(t)| \leq |a_i(0)| \exp(t\|\Sigma\betastar\|+2t\E[\|x\|^2]\lambda^{2-4\varepsilon}).
\end{align*}
In particular, for our choice of $\tau$, for a small enough $\lambdastar=\bigO{d^{-\frac{1}{2-4\varepsilon}}}$, for any $t\leq\min(\tau,t_1)$:
\begin{equation}\label{eq:normearly}
|a_i(t)|< |a_i(0)| \lambda^{-\varepsilon}.
\end{equation}
Note that this implies that $t<t_1$, i.e., $\tau<t_1$. As a consequence. Moreover, we can also show that $|a_i(t)|> |a_i(0)| \lambda^{\varepsilon}$ for any $t\leq \tau$, which implies the first point of \cref{lemma:phase1}.

2. For the second point, let $i\in\cI_+$ and denote $\w_i(t)=\frac{w_i(t)}{a_i(t)}$. Thanks to \cref{lemma:balancedness}, $\w_i(t)\in B(0,1)$ and~$a_i(t)$ is of constant sign. Also, for almost any $t\in[0,\tau]$:
\begin{align*}
\frac{\df \w_i(t)}{\df t} & \in \cD_n(w_i(t), \theta(t)) - \langle \w_i(t), \cD_n(w_i(t),\theta(t)) \rangle \w_i(t).
\end{align*}
Since $a_i(t)>0$ for $i\in\cI_+$,
\begin{align*}
\frac{\df \langle \w_i(t), \Sigma \betastar \rangle}{\df t} & \in \langle \cD_n(w_i(t), \theta(t)), \Sigma \betastar \rangle - \langle \w_i(t), \cD_n(w_i(t),\theta(t)) \rangle \langle \w_i(t), \Sigma \betastar\rangle\\
& \geq \inf_{D_n\in \D_n(w_i(t),\theta(t))} \langle D_n, \Sigma \betastar \rangle - \langle \w_i(t), D_n \rangle \langle \w_i(t), \Sigma \betastar\rangle\\
& \geq \inf_{D_n\in \D_n(w_i(t))} \langle D_n, \Sigma \betastar \rangle - \langle \w_i(t), D_n \rangle \langle \w_i(t), \Sigma \betastar\rangle - 2\|\Sigma\betastar\| \lambda^{2-4\varepsilon} \\
& \geq \langle D(w_i(t)), \Sigma \betastar \rangle - \langle \w_i(t), D(w_i(t)) \rangle \langle \w_i(t), \Sigma \betastar\rangle \\
&\phantom{\geq\geq}- 2\|\Sigma\betastar\| \left(\lambda^{2-4\varepsilon} + \sup_{D_n\in\cD_n(w_i(t)} \|D_n-D(w_i(t))\|\right)\\
& \geq \frac{1}{2}\left(\|\Sigma \betastar\|^2-\langle\w_i(t), \Sigma \betastar \rangle^2\right) -\bigO{\lambda^{2-4\varepsilon} + \sqrt{\frac{d^2\log n}{n}}}.
\end{align*}
Solutions of the ODE $f'(t)=a^2-f(t)^2$ with $f(0)\in (-a,a)$ are of the form $f(t)=a\tanh(a(t+t_0))$ for some $t_0\in\R$.
By Grönwall comparison, we thus have
\begin{gather}
\langle \w_i(t), \Sigma \betastar \rangle \geq a\tanh(\frac{a}{2}(t+t_j)),\label{eq:earlyalign}\\
\text{where}\quad a= \|\Sigma \betastar\|-\bigO{\lambda^{2-4\varepsilon} + \sqrt{\frac{d^2\log n}{n}}}\notag\\
\text{and}\quad  \langle \w_i(0), \Sigma \betastar \rangle = a\tanh(\frac{a}{2}t_j).\notag
\end{gather}

Thanks to the choice of initialization given by \cref{eq:dominatedinit}, $\|\w_i(0)\|\leq \frac{1}{2}$ and so $\langle \w_i(0), \Sigma \betastar \rangle\geq -\frac{1}{2}\|\Sigma \betastar\|_2$. Moreover, $\tanh(x)\leq -1+2e^{2x}$, so that
\begin{align*}
-\frac{1}{2}\|\Sigma \betastar\|\leq a(-1+2e^{a t_j}).
\end{align*}
Since $a= \|\Sigma \betastar\|-\bigO{\lambda^{2-4\varepsilon} + \sqrt{\frac{d^2\log n}{n}}}$, this yields
\begin{align*}
 2ae^{at_j} \geq \frac{1}{2}\|\Sigma \betastar\| +\bigO{\lambda^{2-4\varepsilon} + \sqrt{\frac{d^2\log n}{n}}}.
\end{align*}
The previous inequality can be rewritten as
\begin{align*}
-2ae^{-t_j}&\geq \frac{-4a^2}{\frac{1}{2}\|\Sigma \betastar\|  + \bigO{\lambda^{2-4\varepsilon} + \sqrt{\frac{d^2\log n}{n}}}}\\
&\geq \frac{-8a^2}{\|\Sigma \betastar\|}(1+\bigO{\lambda^{2-4\varepsilon} + \sqrt{\frac{d^2\log n}{n}}})\\
& \geq 8\|\Sigma \betastar\|+ \bigO{\lambda^{2-4\varepsilon} + \sqrt{\frac{d^2\log n}{n}}}.
\end{align*}

Using that $\tanh(x)\geq 1-2e^{-2x}$, \cref{eq:earlyalign} becomes at time $\tau$ and $\frac{n}{\log(n)}=\Omega(d^2)$,
\begin{align*}
\langle \w_i(\tau), \Sigma \betastar \rangle &\geq  a-2ae^{-at_j}e^{-a\tau}\\
&\geq \|\Sigma \betastar\|-\|\Sigma \betastar\|(8\|\Sigma \betastar\| + \bigO{\lambda^{2-4\varepsilon} + \sqrt{\frac{d^2\log n}{n}}})e^{\frac{a\varepsilon\log\lambda}{\|\Sigma\betastar\|}}\\&\phantom{=}- \bigO{\lambda^{2-4\varepsilon} + \sqrt{\frac{d^2\log n}{n}}}\\
&\geq \|\Sigma \betastar\| - \bigO{\lambda^{\varepsilon}+ \sqrt{\frac{d^2\log n}{n}}}.
\end{align*}

3. The same arguments can be done with negative neurons.
\end{proof}

\subsection{Decoupled autonomous systems}

In the remaining of the proof, we will focus on an alternative solution $(\altw,\alta)$, which is solution of the following differential equations for any $t\geq\tau$
\begin{equation}\label{eq:autonomous}
\begin{gathered}
\frac{\df \altw_i(t)}{\df t} = \alta_i(t) D_+(t) \quad\text{and}\quad \frac{\df \alta_i(t)}{\df t} = \langle \altw_i(t), D_+(t)\rangle \quad \text{for any } i\in\cI_+,\\
\frac{\df \altw_i(t)}{\df t} = \alta_i(t) D_-(t) \quad\text{and}\quad \frac{\df \alta_i(t)}{\df t} = \langle \altw_i(t), D_-(t)\rangle \quad \text{for any } i\in\cI_-,
\end{gathered}
\end{equation}
where
\begin{gather*}
D_+(t) = \frac{1}{n}\sum_{k\in\cS_+}\left(\sum_{i\in\cI_+}\alta_i(t)\langle\altw_i(t),x_k\rangle-y_k\right)x_k,\\
D_-(t) = \frac{1}{n}\sum_{k\in\cS_-}\left(\sum_{i\in\cI_-}\alta_i(t)\langle\altw_i(t),x_k\rangle-y_k\right)x_k
\end{gather*}
and with the initial condition $\altw_i(\tau),\alta_i(\tau)=w_i(\tau),a_i(\tau)$ for any $i\in[m]$. 
We also note in the following $\altwn_i=\frac{\altw_i}{\alta_i}$ and the estimations of the training data $x_k$ for any $k\in[n]I$ as:
\begin{equation*}
h_{\alttheta}(x_k) = \begin{cases}\sum_{i\in\cI_+} \alta_i \langle \altw_i,x_k\rangle \quad \text{if } k\in\cS_+\\
\sum_{i\in\cI_-} \alta_i \langle \altw_i,x_k\rangle  \quad \text{if } k\in\cS_-
\end{cases}.
\end{equation*}

This construction allows to study separately the dynamics of both sets of neurons $\cI_+$ and $\cI_-$, without any interaction between each other. As precised by \cref{lemma:autonomous} below, $\altw_i,\alta_i$ coincide with $w_i,a_i$ as long as the neurons all remain in the sector they are at the end of the early alignment phase.
\begin{lem}\label{lemma:autonomous}
Define 
$T_+=\inf\{t\geq \tau \mid \exists (i,k)\in\cI_+\times[n], \sign(x_k^{\top}\altw_i(t)) \neq \sign(x_k^{\top}\betastar)  \}$\\
and $T_-=\inf\{t\geq \tau \mid \exists (i,k)\in\cI_-\times[n], \sign(x_k^{\top}\altw_i(t)) \neq -\sign(x_k^{\top}\betastar)  \}$.

Then for any $i\in[m]$ and any $t\in[\tau,\min(T_+,T_-)]$:
$(\altw_i(t),\alta_i(t))=(w_i(t,a_i(t))$. Moreover, for any $t\in[\tau,\min(T_+,T_-)]$ and $k\in[n]$, $h_{\alttheta(t)}(x_k)= h_{\theta(t)}(x_k)$.
\end{lem}
While analyzing the complete dynamics of $(\altw,\alta)$, we will see that both $T_+$ and $T_-$ are infinite in the considered range of parameters, thus leading to a complete description of the dynamics of $(w,a)$.
\begin{proof}
Thanks to the definition of $T_+$ and $T_-$, the evolution of $(w_i(t),a_i(t))$ given by \ref{eq:ODE} coincides with the evolution of $(\altw_i(t),\alta_i(t))$ given by \cref{eq:autonomous} for $t\in[\tau,\min(T_+,T_-)]$. The associated ODE is Lipschitz on the considered time interval and thus admits a unique solution, hence leading to $(\altw_i(t),\alta_i(t))=(w_i(t,a_i(t))$ on the considered interval. The equality $h_{\alttheta(t)}(x_k)= h_{\theta(t)}(x_k)$ directly derives from the ReLU activations and definitions of $T_+$ and $T_-$.
\end{proof}

\subsection{Phase 2: neurons slow growth}

For some $\varepsilon_2>0$, we define the following stopping time for any $\circ\in\{+,-\}$:
\begin{equation*}
\tau_{2,\circ} = \inf\{t\geq\tau\mid \sum_{i\in\cI_{\circ}}\alta_i(t)^2\geq\varepsilon_2\}.
\end{equation*}

\begin{restatable}{lem}{phase2}\label{lemma:phase2}
If \cref{ass:noconvergence} holds, for any $\varepsilon\in(0,\frac{1}{4})$, there exist $\lambdastar=\Theta(\frac{1}{d})$, $\epsstar_2=\Theta(d^{-\frac{3}{2}})$ and $\nstar=\Theta(d^3\log d)$ such that for any $\lambda\leq\lambdastar$, $n\geq \nstar$, $\circ\in\{+,-\}$, $\varepsilon_2\in[\lambda^{2-4\varepsilon},\epsstar_2]$, with probability $1-\bigO{\frac{1}{n}+\frac{1}{2^m}}$,  $\tau_{2, \circ}<+\infty$ and at this time,
\begin{enumerate}
\item neurons in $\cI_{\circ}$ are aligned with each other 
\begin{equation*}
\forall i,j\in \cI_{\circ}, \quad  \langle \altwn_j(\tau_{2, \circ}), \altwn_i(\tau_{2, \circ})\rangle = 1 - \bigO{\frac{\lambda^{\frac{1}{2}}}{\varepsilon_2}};
\end{equation*}
\item neurons in $\cI_{\circ}$ are in the same cone as $\circ\betastar$ for any $t\in[\tau,\tau_{2,\circ}]$:
\begin{equation*}
\forall i\in \cI_{\circ}, \quad  \min_{k\in\cS_{\circ}}\langle \altwn_i(\tau_{2, \circ}), \frac{x_k}{\|x_k\|}\rangle = \Omega(\frac{1}{\sqrt{d}}) \quad \text{ and }\quad  \max_{k\in\cS_{-\circ}}\langle \altwn_i(\tau_{2, \circ\circ}), \frac{x_k}{\|x_k\|}\rangle = -\Omega(\frac{1}{\sqrt{d}}).
\end{equation*}
\end{enumerate}
\end{restatable}

\begin{proof}
In the following, we assume without loss of generality that $\circ=+$. Additionally, we assume that the random event $\cI_+\neq\emptyset$ and \cref{eq:concentrationlinear,eq:xnorms} hold. First, by definition of $\tau_{2,+}$, for any $t\in[\tau,\tau_{2,+}]$:
\begin{align*}
\|D_+(t) - D_n(\betastar)\|_2 & \leq \frac{1}{n}\left(\sum_{i\in\cI_+}\alta_i(t)^2\right)\sum_{k\in\cS_+}\|x_k\|^2\\
&\leq 2\varepsilon_2\E_{\mu_X}[\|x\|^2].
\end{align*}
This also implies with \cref{eq:normDn} that $\|D_+(t)\|\leq \|\Sigma\betastar\|+2\varepsilon_2\E_{\mu_X}[\|x\|^2]$. Additionally, we have with \cref{eq:concentrationlinear} that
\begin{align}
\|D_+(t) - \frac{\Sigma\betastar}{2}\|_2 & \leq \|D_+(t) - D_n(\betastar)\|_2 + \|D_n(\betastar)-D(\betastar)\|_2\notag\\
&\leq \bigO{d\varepsilon_2+\sqrt{\frac{d^2\log n}{n}}}.
\end{align}
Then for any $k\in\cS_+, i\in\cI_+$ and $t\in[\tau,\tau_{2,+}]$, as long as $\langle \altwn_i(t),x_k\rangle\geq 0$,
\begin{align}
\frac{\df \langle \altwn_i(t),\frac{x_k}{\|x_k\|}\rangle}{\df t} & = \langle D_+(t),\frac{x_k}{\|x_k\|}\rangle - \langle D_+(t),\altwn_i(t)\rangle\langle \altwn_i(t),\frac{x_k}{\|x_k\|}\rangle\notag\\
&\geq \langle D_n(\betastar),\frac{x_k}{\|x_k\|} \rangle -2\varepsilon_2\E_{\mu_X}[\|x\|^2]- \|D_+(t)\|\langle \altwn_i(t),\frac{x_k}{\|x_k\|}\rangle\notag\\
&\geq \langle D_n(\betastar),\frac{x_k}{\|x_k\|} \rangle-2\varepsilon_2\E_{\mu_X}[\|x\|^2]- (\|D_n(\betastar)\|+2\varepsilon_2\E_{\mu_X}[\|x\|^2])\langle \altwn_i(t),\frac{x_k}{\|x_k\|}\rangle.
\end{align}
As $\langle D_n(\betastar),\frac{x_k}{\|x_k\|} \rangle\geq\frac{\alpha}{\sqrt{d}}$ (Equation~\ref{eq:angleDn}), thanks to \cref{lemma:phase1} and the third point in \cref{ass:noconvergence}, and $\|D_n(\betastar)\|=\bigO{1}$, for a small enough $\epsstar_2=\Theta(d^{-\frac{3}{2}})$, $\min_{k\in\cS_+}\langle \altwn_i(\tau),\frac{x_k}{\|x_k\|}\rangle=\Omega(\frac{1}{\sqrt{d}})$. \cref{eq:phase2corr} then implies for a small enough choice of $\epsstar_2=\Theta(d^{-\frac{3}{2}})$ and $\varepsilon_2\leq\epsstar_2$:
\begin{equation}\label{eq:normgrowth}
\min_{t\in[\tau,\tau_{2,+}}\min_{k\in\cS_+}\langle \altwn_i(\tau),\frac{x_k}{\|x_k\|}\rangle=\Omega(\frac{1}{\sqrt{d}}).
\end{equation}
Similarly, we can also show
\begin{equation*}
\max_{t\in[\tau,\tau_{2,+}}\max_{k\in\cS_-}\langle \altwn_i(\tau),\frac{x_k}{\|x_k\|}\rangle=-\Omega(\frac{1}{\sqrt{d}}),
\end{equation*}
which implies the second point of \cref{lemma:phase2}. Actually, we even have for this choice of parameters the more precise inequality (for the same reasons) that for any $k\in\cS_+, i\in\cI_+$ and $t\in[\tau,\tau_{2,+}]$,
\begin{equation}\label{eq:phase2corr}
\langle \altwn_i(\tau),\frac{x_k}{\|x_k\|}\rangle\geq \langle \frac{D_n(\betastar)}{\|D_n(\betastar)\|},\frac{x_k}{\|x_k\|} \rangle -\bigO{d\varepsilon_2}.
\end{equation}

\medskip

We now simultaneously lower and upper bound the duration of the second phase $\tau_{2,+}-\tau_2$. For any $t\in[\tau,\tau_{2,+}]$:
\begin{align}
\frac{1}{2}\frac{\df \sum_{i\in\cI_+}\alta_i(t)^2}{\df t} & = \sum_{i\in\cI_+}\alta_i(t)^2 \langle\altwn_i(t),D_+(t)\rangle\notag\\
& = \frac{1}{n}\sum_{i\in\cI_+}\alta_i(t)^2 \sum_{k\in\cS_+}(y_k-h_{\alttheta(t)}(x_k))\langle \altwn_i(t),x_k\rangle\label{eq:normgrowth2}\\
& \geq \sum_{i\in\cI_+}\alta_i(t)^2 \left(\frac{ \sum_{k\in\cS_+}y_k \|x_k\|\langle \altwn_i(t),\frac{x_k}{\|x_k\|}\rangle}{n} - \frac{\varepsilon_2 \sum_{k\in\cS_+}\|x_k\|^2}{n} \right).\notag
\end{align}
Note that $\E[\iind{k\in\cS_+}y_k\|x_k\|]\geq\frac{c\E[\|x\|^2]}{2\sqrt{d}}$. Using Chebyshev inequality, we thus have for a small enough choice of $\epsstar_2=\Theta(d^{-\frac{3}{2}})$, for any $t\in[\tau,\tau_{2,+}]$:
\begin{align*}
\frac{\df \sum_{i\in\cI_+}\alta_i(t)^2}{\df t} \geq \Omega(1)\sum_{i\in\cI_+}\alta_i(t)^2.
\end{align*}
A Gr\"onwall comparison then directly yields $\tau_{2,+}<\infty$.

We now want to show that the neurons $\altwn_i$ are almost aligned at the end of the second phase. For that, we first need to lower bound the duration of the phase. Note that \cref{eq:normgrowth2}, with \cref{eq:normgrowth}, also leads for any $t\in[\tau,\tau_{2,+}]$ to
\begin{align*}
\frac{1}{2}\frac{\df \sum_{i\in\cI_+}\alta_i(t)^2}{\df t} & \leq \frac{1}{n}\sum_{i\in\cI_+}\alta_i(t)^2 \sum_{k\in\cS_+}y_k\langle \altwn_i(t) ,x_k\rangle \\
& = \sum_{i\in\cI_+}\alta_i(t)^2 \langle\altwn_i(t) ,D_n(\betastar)\rangle\\
& \leq \sum_{i\in\cI_+}\alta_i(t)^2  \|D_n(\betastar)\|.
\end{align*}
Note that by continuity, $\sum_{i\in\cI_+}\alta_i(\tau_{2,+})^2=\varepsilon_2$. 
As $\sum_{i\in\cI_+}\alta_i(\tau)^2\leq \lambda^{2-4\varepsilon}$, thanks to \cref{lemma:phase1}, a Gr\"onwall inequality argument leads to the following as $\varepsilon_2\geq\lambda^{2-4\varepsilon}$,
\begin{equation}\label{eq:phase2duration}
\tau_{2,+}-\tau \geq \frac{1}{2\|D_n(\betastar)\|}\ln\left(\frac{\varepsilon_2}{\lambda^{2-4\varepsilon}}\right).
\end{equation} 

For any pair of neurons $i,j\in\cI_+$, we consider the evolution of the mutual alignment:
\begin{align}
\frac{\df \langle \altwn_i(t), \altwn_j(t) \rangle}{\df t} & = \langle D_+(t), \altwn_i(t) + \altwn_j(t)\rangle (1-\langle \altwn_i(t), \altwn_j(t) \rangle)\notag\\
& = \left(\langle D_n(\betastar), \altwn_i(t) + \altwn_j(t)\rangle - \bigO{\varepsilon_2}\right)(1-\langle \altwn_i(t), \altwn_j(t) \rangle).\label{eq:phase2mutual}
\end{align}
Moreover, \cref{eq:phase2corr} leads to the following alignment between $\altwn_i(t)$ and $D_n(\betastar)$ for any $t\in[\tau,\tau_{2,+}]$:
\begin{align*}
\langle D_n(\betastar), \altwn_i(t) \rangle & = \frac{1}{n}\sum_{k\in\cS_+} y_k \langle x_k,\altwn_i(t) \rangle \\
& = \frac{1}{n}\sum_{k\in\cS_+} y_k \langle x_k, \frac{D_n(\betastar)}{\|D_n(\betastar)\|}\rangle - \|x_k\|\bigO{d\varepsilon_2}\\
& = \langle D_n(\betastar), \frac{D_n(\betastar)}{\|D_n(\betastar)\|}\rangle-\bigO{d\varepsilon_2}\\
& = \|D_n(\betastar)\|-\bigO{d\varepsilon_2}.
\end{align*}
\cref{eq:phase2mutual} then rewrites for any $t\in[\tau,\tau_{2,+}]$ as
\begin{equation*}
\frac{\df \langle \altwn_i(t), \altwn_j(t) \rangle}{\df t}  \geq \left(2\|D_n(\betastar)\| - \bigO{d\varepsilon_2}\right) (1-\langle \altwn_i(t), \altwn_j(t) \rangle).
\end{equation*}
Moreover, thanks to \cref{lemma:phase1}, a simple algebraic manipulation yields\footnote{A similar manipulation can be found in \citep[][proof of Lemma 5]{boursier2024early}.} $\langle \altwn_i(\tau), \altwn_j(\tau) \rangle \geq 1-\bigO{\lambda^\varepsilon+\sqrt{\frac{d^2\log n}{n}}}$. Gr\"onwall inequality then yields, for the considered range of parameters,
\begin{align*}
\langle \altwn_i(\tau_{2,+}), \altwn_j(\tau_{2,+}) \rangle & \geq 1 - (1-\langle \altwn_i(\tau), \altwn_j(\tau) \rangle)\, e^{-(2\|D_n(\betastar)\|- \bigO{d\varepsilon_2}) (\tau_{2,+}-\tau)}\\
& \geq 1 - \bigO{\lambda^{\varepsilon}+\sqrt{\frac{d^2\log n}{n}}}\frac{\lambda^{2-4\varepsilon}}{\varepsilon_2}e^{\bigO[\ ]{d\varepsilon_2\ln(\frac{\varepsilon_2}{\lambda^{2-4\varepsilon}})}}\\
&\geq 1-\bigO{\frac{\lambda}{\varepsilon_2}}\lambda^{-(2-4\varepsilon)\bigO{d\varepsilon_2}}.
\end{align*}
The second inequality comes from the bound on $\tau_{2,+}-\tau$ in \cref{eq:phase2duration}. The third one comes from the fact that $\varepsilon\leq\frac{1}{3}$ and $\varepsilon_2\ln(\varepsilon_2)=\bigO{1}$. Noticing that $2-4\varepsilon\geq1$ finally yields the first item of \cref{lemma:phase2} for a small enough $\epsstar_2=\Theta(d^{-\frac{3}{2}})$.
\end{proof}

\subsection{Phase 3: neurons fast growth}
The third phase is defined for some $\varepsilon_3>0$ and $\delta_3$ by the following stopping time, for any $\circ\in\{+,-\}$:
\begin{gather*}
\tau_{3,\circ} = \inf\{t\geq\tau_{2,\circ} \mid \|\hbeta_{\circ}(t)-\beta_{n,\circ}\|_{\Sigma_{n,\circ}} \leq \varepsilon_3 \text{ or } \exists i\in\cI_\circ,k\in\cS_\circ, \langle\altw_i(t),\frac{x_k}{\|x_k\|}\rangle\leq \delta_3\},\\
\text{where}\quad \hbeta_{\circ}(t) = \sum_{i\in\cI_\circ} \alta_i(t)\altw_i(t).
\end{gather*}
\begin{restatable}{lem}{phase3}\label{lemma:phase3}
If \cref{ass:noconvergence} holds, for any $\varepsilon\in(0,\frac{1}{4})$, there exist $\lambdastar=\Theta(\frac{1}{d})$, $\epsstar_2=\Theta(d^{-\frac{3}{2}})$, $\nstar=\Theta(d^3\log d)$, $\alpha_0=\Theta(1)$, $\delta_3=\Theta(\frac{1}{\sqrt{d}})$ and $\epsstar_3=\Theta(1)$ such that for any $\lambda\leq\lambdastar$, $n\geq \nstar$, $\circ\in\{+,-\}$, $\varepsilon_2\in[\lambda^{2-4\varepsilon},\epsstar_2]$ and $\varepsilon_3\in[\lambda^{\alpha_0\varepsilon\varepsilon_2},\epsstar_3]$, with probability $1-\bigO{\frac{d^2}{n}+\frac{1}{2^m}}$,  $\tau_{3, \circ}<+\infty$ and
\begin{enumerate}
\item neurons in $\cI_{\circ}$ are in the same cone as $\circ\betastar$ for any $t\in[\tau,\tau_{2,\circ}]$:
\begin{equation*}
\forall i\in \cI_{\circ}, \quad  \min_{k\in\cS_{\circ}}\langle \altwn_i(t), \frac{x_k}{\|x_k\|}\rangle \geq 2\delta_3 \quad \text{ and }\quad  \max_{k\in\cS_{-\circ}}\langle \altwn_i(t), \frac{x_k}{\|x_k\|}\rangle \leq -2\delta_3.
\end{equation*}
\end{enumerate}
In particular,  $\|\hbeta_{\circ}(\tau_{3,\circ})-\beta_{n,\circ}\|_{\Sigma_{n,\circ}} = \varepsilon_3$ by continuity.
\end{restatable}

\begin{proof}
Similarly to the proof of \cref{lemma:phase2}, we assume that $\circ=+$, that the random event $\cI_+\neq\emptyset$, \cref{eq:concentrationlinear,eq:xnorms} and the first and second items states in \cref{lemma:phase2} all hold. 
We can first show that for any $t\in[\tau_{2,+},\tau_{3,+}]$,
\begin{equation*}
\sum_{i\in\cI_+} \alta_i(t)^2\geq \varepsilon_2.
\end{equation*}
Indeed, recall that the output weights $\alta_i$ evolve for any $t\in[\tau_{2,+},\tau_{3,+}]$ as
\begin{align}
\frac{\df \alta_i(t)}{\df t} & =\langle\altw_i(t),D_+(t)\rangle\notag \\
& = \langle\altw_i(t),D_n(\betastar)\rangle - \frac{1}{n}\sum_{k\in\cS_+} h_{\alttheta(t)}(x_k)\langle \altw_i(t),x_k\rangle\notag\\
& \geq \alta_i(t)\left(\frac{1}{n}\sum_{k\in\cS_+}\langle\altwn_i(t),x_k\rangle \langle\betastar,x_k\rangle + \frac{1}{n}\sum_{k\in\cS_+}\langle\altwn_i(t),\eta_k x_k\rangle - \bigO{d\sum_{i\in\cI_+}\alta_i(t)^2} \right).\label{eq:phase3rate0}
\end{align}
The last inequality comes from the fact that $\frac{\sum_{k\in\cS_+}\|x_k\|^2}{n}=\bigO{d}$. From then, note that $\frac{1}{n}\sum_{k\in\cS_+}\langle\altwn_i(t),x_k\rangle \langle\betastar,x_k\rangle\geq \Omega(\delta_3\sqrt{d})$ during this phase. Moreover, using Chebyshev inequality, we can show for any $z>0$ that with probability at least $1-\bigO{\frac{d}{z^2 n}}$
\begin{align}\label{eq:concresiduals}
\frac{1}{n}\left\|\sum_{k\in\cS_+}\eta_k x_k \right\|_2 \leq z.
\end{align}
Taking a small enough $z=\Theta(\delta_3)$, \cref{eq:concresiduals} holds with probability $1-\bigO{\frac{d}{\delta_3^2 n}}$ and, along \cref{eq:phase3rate0}, this implies that for any $t\in[\tau_{2,+},\tau_{3,+}]$ and $i\in\cI_+$:
\begin{align*}
\frac{\df \alta_i(t)}{\df t} \geq \alta_i(t) \left(\Omega(\delta_3) -\bigO{d\sum_{i\in\cI_+}\alta_i(t)^2}\right).
\end{align*}
In particular, there exists $r=\Theta(\frac{\delta_3}{d})$ such that if $\sum_{i\in\cI_+}\alta_i(t)^2\leq r$, all the $\alta_i(t)$ are increasing.
Moreover thanks to \cref{lemma:phase2}, $\sum_{i\in\cI_+}\alta_i(\tau_{2,+})^2= \varepsilon_2$. As $\delta_3=\Theta(\frac{1}{\sqrt{d}})$, we can choose $\epsstar_2=\Theta(d^{-\frac{3}{2}})$ small enough so that during the third phase,
\begin{equation}\label{eq:phase3rate}
\sum_{i\in\cI_+}\alta_i(t)^2 \geq \varepsilon_2.
\end{equation}
Now note that by definition of $\beta_{n,+}$,
\begin{align}
D_+(t)& = -\frac{1}{n}\sum_{k\in\cS_+} x_k x_k^{\top}\hbeta_+(t) - x_k y_k \notag\\
& = -\Sigma_{n,+}(\hbeta_+(t)-\beta_{n,+})\label{eq:D+}
\end{align}

As a consequence, $\hbeta_+(t)$ evolves as follows:
\begin{align*}
\frac{\df \hbeta_+(t)}{\df t} & = \sum_{i\in\cI_+} \left(\alta_i(t)^2 \bI_d + \altw_i(t)\altw_i(t)^{\top}\right) D_+(t) \\
& = -\left( \sum_{i\in\cI_+} \alta_i(t)^2 \bI_d + \sum_{i\in\cI_+} \altw_i(t)\altw_i(t)^{\top}\right) \Sigma_{n,+}(\hbeta_+(t)-\beta_{n,+})
\end{align*}
In particular, this implies:
\begin{align}
\frac{1}{2}\frac{\df \|\hbeta_{+}(t)-\beta_{n,+}\|_{\Sigma_{n,+}}^2}{\df t} & = \left\langle \frac{\df \hbeta_+(t)}{\df t}, \Sigma_{n,+}(\hbeta_+(t)-\beta_{n,+}) \right\rangle\\
& = - (\hbeta_+(t)-\beta_{n,+})^{\top}  \Sigma_{n,+}\left( \sum_{i\in\cI_+} \alta_i(t)^2 \bI_d + \sum_{i\in\cI_+} \altw_i(t)\altw_i(t)^{\top}\right) \Sigma_{n,+}(\hbeta_+(t)-\beta_{n,+}).\label{eq:phase3growthrate2}
\end{align}
The matrix $\Sigma_{n,+}^{\nicefrac{1}{2}}\left( \sum_{\cI_+} \alta_i(t)^2 \bI_d + \sum_{\cI_+} \altw_i(t)\altw_i(t)^{\top}\right) \Sigma_{n,+}^{\nicefrac{1}{2}}$ is symmetric, positive definite. Thanks to \cref{eq:phase3rate}, its smallest eigenvalue is larger than $\varepsilon_2\lambda_{\min}(\Sigma_{n,+})$, where $\lambda_{\min}(\cdot)$ denotes the smallest eigenvalue of a matrix. Using typical concentration inequalities on the empirical covariance \citep[see e.g.][Section 4.7]{vershynin2018high}, with probability $1-\bigO{\frac{1}{n}}$, $\|\Sigma_{n,+}-\frac{\Sigma}{2}\|_{\op}=\bigO{\sqrt{\frac{d+\log n}{n}}}$.
With the fourth point in \cref{ass:noconvergence}, we then have for a large enough $\nstar=\Theta(d^3\log d)$ and with probability $1-\bigO{\frac{1}{n}}$,
\begin{equation}\label{eq:covarianceeigenvalue}
\begin{gathered}
\bigO{1}\geq\lambda_{\max}(\Sigma_{n,+})\geq \lambda_{\min}(\Sigma_{n,+}) \geq \Omega(1)\\
\text{and}\quad\frac{\lambda_{\max}(\Sigma_{n,+})}{\lambda_{\min}(\Sigma_{n,+})} = \frac{\lambda_{\max}(\Sigma)}{\lambda_{\min}(\Sigma)}+\bigO{\sqrt{\frac{d+\log n}{n}}},
\end{gathered}
\end{equation}
where $\lambda_{\max}(\cdot)$ denotes the largest eigenvalue.

Assume \cref{eq:covarianceeigenvalue} holds in the following, so that the smallest eigenvalue of $\Sigma_{n,+}^{\nicefrac{1}{2}}\left( \sum_{\cI_+} \alta_i(t)^2 \bI_d + \sum_{\cI_+} \altw_i(t)\altw_i(t)^{\top}\right)\Sigma_{n,+}^{\nicefrac{1}{2}}$ is larger than a term of order $\varepsilon_2$. As a consequence, \cref{eq:phase3growthrate2} yields
\begin{equation*}
\frac{1}{2}\frac{\df \|\hbeta_{+}(t)-\beta_{n,+}\|_{\Sigma_{n,+}}^2}{\df t} \leq -\Omega(\varepsilon_2)\|\hbeta_{+}(t)-\beta_{n,+}\|_{\Sigma_{n,+}}^2.
\end{equation*}
Since the third phase ends if $\|\hbeta_{+}(t)-\beta_{n,+}\|_{\Sigma_{n,+}}^2$ becomes smaller than $\varepsilon_3^2$, this yields:
\begin{equation}\label{eq:p3duration}
\tau_{3,+}-\tau_{2,+} = \bigO{\frac{1}{\varepsilon_2}\ln(\frac{1}{\varepsilon_3})}.
\end{equation}
Now recall that for any $i,j\in\cI_+$,
\begin{align*}
\frac{\df (1- \langle \altwn_i(t), \altwn_j(t) \rangle)}{\df t} & = -\langle D_+(t), \altwn_i(t) + \altwn_j(t)\rangle (1-\langle \altwn_i(t), \altwn_j(t) \rangle)\\
& \leq 2\|D_+(t)\|_2 (1-\langle \altwn_i(t), \altwn_j(t) \rangle).
\end{align*}
Notice from \cref{eq:D+} and the previous discussion that $\|D_+(t)\|_2 = \bigO{1}$. As a consequence, a simple Gr\"onwall inequality with \cref{eq:p3duration} yields that for any $t\in[\tau_{2,+},\tau_{3,+}]$:
\begin{align*}
\langle \altwn_i(t), \altwn_j(t) \rangle & \geq 1- (1-\langle \altwn_i(\tau_{2,+}), \altwn_j(\tau_{2,+}) \rangle)\exp((t-\tau_{2,+})\bigO{1})\notag\\
& \geq 1 - \bigO{\frac{\lambda^{\frac{1}{2}}}{\varepsilon_2}}\exp\left(\bigO{\frac{1}{\varepsilon_2}\ln(\frac{1}{\varepsilon_3})}\right) \notag\\
& \geq  1 - \bigO{\lambda^{\frac{1}{2}-\varepsilon}}.\label{eq:phase3alignment}
\end{align*}
The second inequality comes from the value of $(1-\langle \altwn_i(\tau_{2,+}), \altwn_j(\tau_{2,+}) \rangle)$, thanks to \cref{lemma:phase2}. The last one comes from our choice of $\varepsilon_3$ for a large enough $\alpha_0=\Theta(1)$.

In particular, this last inequality can be used to show\footnote{For that, we decompose  $\altwn_i=\alpha_{ij}\altwn_j+u_{ij}$ with $u_{ij}\perp\altwn_j$ and show that $\alpha_{ij}=1-\bigO{\lambda^{\frac{1}{2}-\varepsilon}}$ and $\|u_{ij}\|^2=\bigO{\lambda^{\frac{1}{2}-\varepsilon}}$.} that for any $i,j \in\cI_+$ and $t\in[\tau_{2,+},\tau_{3,+}]$,  $\altwn_i(t)=\altwn_j(t) + \bigO{\lambda^{\frac{1-2\varepsilon}{4}}}$. 
In particular, this yields for any $i\in\cI_+$ and $t\in[\tau_{2,+},\tau_{3,+}]$
\begin{align}
\hbeta_+(t) & = \sum_{j\in\cI_+}\alta_j(t)^2\altwn_j(t) \\
& = \left(\sum_{j\in\cI_+}\alta_j(t)^2\right)\left(\altwn_i(t) + \bigO{\lambda^{\frac{1-2\varepsilon}{4}}}\right).\label{eq:phase3tech3}
\end{align}
Since $\|\altwn_i(t)\|_2=1-\bigO{\lambda^{\frac{1}{2}-\varepsilon}}$, this last equality actually yields the following comparison for $t\in[\tau_{2,+},\tau_{3,+}]$:
\begin{equation}\label{eq:normbeta+}
\|\hbeta_+(t)\|_2 \leq  \sum_{j\in\cI_+}\alta_j(t)^2 \leq (1+\bigO{\lambda^{\frac{1-2\varepsilon}{4}}})\|\hbeta_+(t)\|_2.
\end{equation}
In particular, since $\|\hbeta_+(t)\|_2 = \bigO{1}$, this yields $\sum_{j\in\cI_+}\alta_j(t)^2= \bigO{1}$. 

From there, for any $x_k\in\cS_+$ and $i\in\cI_+$, $\langle\altwn_{i}(t),x_k \rangle$ evolves as follows during the third phase
\begin{align*}
\frac{\df\langle\altwn_{i}(t),\frac{x_k}{\|x_k\|} \rangle}{\df t} &  = \langle D_+(t),\frac{x_k}{\|x_k\|}\rangle - \langle D_+(t),\altwn_i(t)\rangle\langle\altwn_i(t),\frac{x_k}{\|x_k\|}\rangle\\
& =  \langle \beta_{n,+}-\hbeta_+(t),\Sigma_{n,+}\frac{x_k}{\|x_k\|}\rangle - \bigO{\langle\altwn_i(t),\frac{x_k}{\|x_k\|}\rangle}\\
& =  \frac{1}{2}\langle \beta_{n,+}-\hbeta_+(t),\frac{x_k}{\|x_k\|}\rangle  + \langle (\Sigma_{n,+}-\frac{\bI_d}{2})(\beta_{n,+}-\hbeta_+(t)),\frac{x_k}{\|x_k\|}\rangle - \bigO{\langle\altwn_i(t),\frac{x_k}{\|x_k\|}\rangle}.
\end{align*}
Note that
\begin{align*}
\beta_{n,+} & =\betastar + \frac{\Sigma_{n,+}^{-1}}{n}\sum_{k\in\cS_+}\eta_k x_k.
\end{align*}
This then yields, thanks to \cref{eq:covarianceeigenvalue}
\begin{multline}\label{eq:phase3angle0}
\frac{\df\langle\altwn_{i}(t),\frac{x_k}{\|x_k\|} \rangle}{\df t} \geq \frac{1}{2}\langle \betastar-\hbeta_+(t),\frac{x_k}{\|x_k\|}\rangle  + \langle (\Sigma_{n,+}-\frac{\bI_d}{2})(\beta_{n,+}-\hbeta_+(t)),\frac{x_k}{\|x_k\|}\rangle\\ - \bigO{\frac{1}{n}\|\sum_{k\in\cS_+}\eta_k x_k\|_2} - \bigO{\langle\altwn_i(t),\frac{x_k}{\|x_k\|}\rangle}.
\end{multline}
From there, thanks to the third point of \cref{ass:noconvergence} and \cref{eq:phase3tech3}:
\begin{align}
\langle \betastar-\hbeta_+(t),\frac{x_k}{\|x_k\|}\rangle & \geq \frac{c}{\sqrt{d}} - \bigO{\langle\altwn_i(t),\frac{x_k}{\|x_k\|}\rangle} - \bigO{\lambda^{\frac{1-2\varepsilon}{4}}}.\label{eq:phase3tech0}
\end{align}
Additionally, using the fact that $\|\beta_{n,+}-\hbeta_+(t)\|_{\Sigma_{n,+}}$ is decreasing over time and smaller than $\|\beta_{n,+}\|_{\Sigma_{n,+}}+\bigO{\lambda^{2-4\varepsilon}}$ at the beginning of the second phase,
\begin{align*}
\langle (\Sigma_{n,+}-\frac{\bI_d}{2})(\beta_{n,+}-\hbeta_+(t)),\frac{x_k}{\|x_k\|}\rangle &\geq -\left\|\Sigma_{n,+}-\frac{\bI_d}{2}\right\|_{\op}\|\beta_{n,+}-\hbeta_+(t)\|_2\\
&\geq -\frac{1}{2}\left\|\Sigma-\bI_d\right\|_{\op}\sqrt{\frac{1}{\lambda_{\min}(\Sigma_{n,+})}}\|\beta_{n,+}-\hbeta_+(t)\|_{\Sigma_{n,+}}-\bigO{\sqrt{\frac{d+\log n}{n}}}\\
&\geq-\frac{1}{2}\left\|\Sigma-\bI_d\right\|_{\op}\sqrt{\frac{1}{\lambda_{\min}(\Sigma_{n,+})}}\left(\|\beta_{n,+}\|_{\Sigma_{n,+}}+\bigO{\lambda^{2-4\varepsilon}}\right)-\bigO{\sqrt{\frac{d+\log n}{n}}}\\
&\geq-\frac{1}{2}\left\|\Sigma-\bI_d\right\|_{\op}\sqrt{\frac{\lambda_{\max}(\Sigma_{n,+})}{\lambda_{\min}(\Sigma_{n,+})}}\|\beta_{n,+}\|_{2} - \bigO{\lambda^{2-4\varepsilon}+\sqrt{\frac{d+\log n}{n}}}\\
&\geq-\frac{1}{2}\sqrt{\frac{\lambda_{\max}(\Sigma_{n,+})}{\lambda_{\min}(\Sigma_{n,+})}}\left\|\Sigma-\bI_d\right\|_{\op}\|\betastar\|_{2} \\
&\phantom{\geq}\quad- \bigO{\lambda^{2-4\varepsilon}+\frac{1}{n}\|\sum_{k\in\cS_+}\eta_k x_k\|_2+\sqrt{\frac{d+\log n}{n}}}  
\end{align*}
Now using \cref{eq:covarianceeigenvalue} and the fourth point of \cref{ass:noconvergence}, note that
\begin{align*}
\sqrt{\frac{\lambda_{\max}(\Sigma_{n,+})}{\lambda_{\min}(\Sigma_{n,+})}}\leq 2+\bigO{\sqrt{\frac{d+\log n}{n}}}.
\end{align*}
So that the previous inequality yields
\begin{equation}\label{eq:phase3tech1}
\langle (\Sigma_{n,+}-\frac{\bI_d}{2})(\beta_{n,+}-\hbeta_+(t)),\frac{x_k}{\|x_k\|}\rangle\geq - \left\|\Sigma-\bI_d\right\|_{\op}\|\betastar\|_{2} -\bigO{\lambda^{2-4\varepsilon}+\frac{1}{n}\|\sum_{k\in\cS_+}\eta_k x_k\|_2+\sqrt{\frac{d+\log n}{n}}}.
\end{equation}
Finally, thanks to \cref{eq:concresiduals}, $\frac{1}{n}\|\sum_{k\in\cS_+}\eta_k x_k\|_2\leq z'$ with probability at least $1-\bigO{\frac{d}{z'^2n}}$. 
Using \cref{eq:phase3tech0,eq:phase3tech1} in \cref{eq:phase3angle0} finally yields for the third phase:
\begin{align*}
\frac{\df\langle\altwn_{i}(t),\frac{x_k}{\|x_k\|} \rangle}{\df t}  & \geq \frac{c}{2\sqrt{d}}-\|\Sigma-\bI_d\|_{\op}\|\betastar\|_2-\bigO{\lambda^{\frac{1-2\varepsilon}{4}}+z'+\sqrt{\frac{d+\log n}{n}}} - \bigO{\langle\altwn_i(t),\frac{x_k}{\|x_k\|}\rangle}.
\end{align*}
Thanks to the fourth point of \cref{ass:noconvergence}, $\frac{c}{2\sqrt{d}}-\|\Sigma-\bI_d\|_{\op}\|\betastar\|_2>0$, so that we can choose $\lambdastar,z'=\Theta(1)$ small enough and $\nstar=\Theta(d^3\log d)$ large enough so that the previous inequality becomes, with probability at least $1-\bigO{\frac{d}{n}}$ 
\begin{align*}
\frac{\df\langle\altwn_{i}(t),\frac{x_k}{\|x_k\|} \rangle}{\df t}  & \geq \Omega(\frac{1}{\sqrt{d}}) - \bigO{\langle\altwn_i(t),\frac{x_k}{\|x_k\|}\rangle}.
\end{align*}
A simple Gr\"onwall argument with the second point of \cref{lemma:phase2} then implies that for any $t\in[\tau_{2,+},\tau_{3,+}]$, $i\in\cI_+$ and $k\in\cS_+$,
\begin{align*}
\langle\altwn_{i}(t),\frac{x_k}{\|x_k\|} \rangle \geq \Omega(\frac{1}{\sqrt{d}}).
\end{align*}
Since the term  $\Omega(\frac{1}{\sqrt{d}})$ here does not depend on $\delta_3$, we can choose $\delta_3=\Theta(\frac{1}{\sqrt{d}})$ small enough so that
\begin{align*}
\langle\altwn_{i}(t),\frac{x_k}{\|x_k\|} \rangle \geq 2\delta_3.
\end{align*}
We can show similarly for $k\in\cS_-$, so that point 1 in \cref{lemma:phase3} holds, which concludes the proof.
\end{proof}

\subsection{Phase 4: final convergence}

The last phase is defined for some $\varepsilon_4>\varepsilon_3$ by the following stopping time, for any $\circ\in\{+,-\}$:
\begin{gather*}
\tau_{4,\circ} = \inf\{t\geq\tau_{3,\circ} \mid \|\hbeta_{\circ}(t)-\hbeta_{\circ}(\tau_{3,\circ})\|_{\Sigma_{n,\circ}} \geq \varepsilon_4\}.
\end{gather*}

\begin{restatable}{lem}{phase4}\label{lemma:phase4}
If \cref{ass:noconvergence} holds, for any $\varepsilon\in(0,\frac{1}{4})$, there exist $\lambdastar=\Theta(\frac{1}{d})$, $\epsstar_2=\Theta(d^{-\frac{3}{2}})$, $\nstar=\Theta(d^3\log d)$, $\alpha_0=\Theta(1)$, $\delta_3=\Theta(\frac{1}{\sqrt{d}})$, $\epsstar_3=\Theta(\frac{1}{\sqrt{d}})$ and $\varepsilon_4=\Theta(\epsstar_3)$ such that for any $\lambda\leq\lambdastar$, $n\geq \nstar$, $\circ\in\{+,-\}$, $\varepsilon_2\in[\lambda^{2-4\varepsilon},\epsstar_2]$ and $\varepsilon_3\in[\lambda^{\alpha_0\varepsilon\varepsilon_2},\epsstar_3]$, with probability $1-\bigO{\frac{d^2}{n}+\frac{1}{2^m}}$,  $\tau_{4, \circ}=+\infty$ and
\begin{enumerate}
\item neurons in $\cI_{\circ}$ are in the same cone as $\circ\betastar$ for any $t\geq\tau_{3,\circ}$:
\begin{equation*}
\forall i\in \cI_{\circ}, \quad  \min_{k\in\cS_{\circ}}\langle \altwn_i(t), \frac{x_k}{\|x_k\|}\rangle >0 \quad \text{ and }\quad  \max_{k\in\cS_{-\circ}}\langle \altwn_i(t), \frac{x_k}{\|x_k\|}\rangle <0.
\end{equation*}
\item $\lim_{t\to\infty}\alttheta(t)$ exists and $\lim_{t\to\infty}\hbeta_{\circ}(t)=\beta_{n,\circ}$.
\end{enumerate}
\end{restatable}
\begin{proof}
Similarly to the previous phases, we assume that $\circ=+$, that the random event $\cI_+\neq\emptyset$, \cref{eq:concentrationlinear,eq:xnorms,eq:covarianceeigenvalue} and the statements of \cref{lemma:phase3} all hold. 

Define in the following the positive loss $L_+$ for any $\alttheta_+\in\R^{(d+1)\times\cI_+}$ by
\begin{equation*}
L_+(\alttheta)=\frac{1}{2n}\sum_{k\in\cS_+} \left(\sum_{i\in\cI+}\alta_i \langle\altw_i,x_k\rangle-y_k\right)^2.
\end{equation*}
Note that the autonomous system given by \cref{eq:autonomous} actually defines a gradient flow over $L_+$, i.e., for $\alttheta_+=(\alta_i,\altw_i)_{i\in\cI_+}$,
\begin{align*}
\frac{\df \alttheta_+(t)}{\df t} = -\nabla L_+(\alttheta_+(t)).
\end{align*}
The main argument for this phase is to prove a local Polyak-\L ojasiewicz inequality:
\begin{gather}\label{eq:PL}
\|\nabla L_+(\alttheta_+)\|_2^2 \geq \Omega(1) (L_+(\alttheta_+)-L_{n,+}) \\
\notag\text{for any } \alttheta_+ \text{ such that}\quad \|\sum_{i\in\cI_+}\alta_i\altw_i-\hbeta_{+}(\tau_{3,\circ})\|_{\Sigma_{n,+}} \leq \varepsilon_4,\\
\notag \text{where}\quad L_{n,+} = \frac{1}{2n}\sum_{k\in\cS_+} \left( \langle\beta_{n,+},x_k\rangle-y_k\right)^2.
\end{gather}

Indeed, we can lower bound $\|\nabla L_+(\alttheta_+)\|_2$ for any such $\alttheta_+$ as follows
\begin{align*}
\|\nabla L_+(\alttheta_+)\|_2^2 & \geq \sum_{i\in\cI_+} \left\|\frac{\partial  L_+(\alttheta_+)}{\partial \altw_i}\right\|^2\\
& = \left(\sum_{i\in\cI_+} \alta_i(t)^2 \right) \|D_+(t)\|^2_2\\
& \geq \lambda_{\min}(\Sigma_{n,+}) \left(\sum_{i\in\cI_+} \alta_i(t)^2 \right) \|\hbeta_+-\beta_{n,+}\|^2_{\Sigma_{n,+}}
\end{align*}
where $\hbeta_+=\sum_{i\in\cI_+}\alta_i\altw_i$. The last inequality comes from \cref{eq:D+}. Note that for a small enough choice of $\epsstar_3=\bigO{1}$ and $\varepsilon_4=\Theta(\epsstar_3)$, $\sum_{i\in\cI_+} \alta_i(t)^2 =\Omega(1)$ in the considered set. Moreover, \cref{eq:covarianceeigenvalue} implies $\lambda_{\min}(\Sigma_{n,+})=\Omega(1)$, so that
\begin{equation}\label{eq:PLaux1}
\|\nabla L_+(\alttheta_+)\|_2^2 \geq\Omega(1) \|\hbeta_+-\beta_{n,+}\|^2_{\Sigma_{n,+}}.
\end{equation}
On the other hand, a simple algebraic manipulation yields for any $\alttheta_+$:
\begin{align*}
L_+(\alttheta_+)-L_{n,+} & =\frac{1}{2n} \sum_{k\in\cS_+}\left(\langle \hbeta_+, x_k \rangle-y_k\right)^2  - \left(\langle \beta_{n,+}, x_k \rangle-y_k\right)^2 \\
& = \frac{1}{2n} \sum_{k\in\cS_+}\left(\hbeta_+-\beta_{n,+}\right)^{\top} x_k  - \left(x_k^{\top} (\hbeta_+ -\beta_{n,+}+2\beta_{n,+})-2y_k\right) \\
& =\frac{1}{2} \left(\hbeta_+-\beta_{n,+}\right)^{\top} \Sigma_{n,+}\left(\hbeta_+-\beta_{n,+}\right)  + \frac{1}{n} \bX_{n,+} ( \bX_{n,+}^{\top} \beta_{n,+} - \by),
\end{align*}
where $\bX_{n,+}$ is the $|\cS_+|\times d$ matrix, whose rows are given by $x_k$ for $k\in\cS_+$. By definition of the OLS estimator $\beta_{n,+}$, $\bX_{n,+}^{\top} \beta_{n,+} - \by = \mathbf{0}$, so that
\begin{equation}\label{eq:PLaux2}
L_+(\alttheta_+)-L_{n,+}  = \frac{1}{2}\|\hbeta_+-\beta_{n,+}\|_{\Sigma_{n,+}}^2.
\end{equation}
Combining \cref{eq:PLaux1} with \cref{eq:PLaux2} finally yields the Polyak-\L ojasiewicz inequality given by \cref{eq:PL}.

From there, this implies by chain rule for any $t\in[\tau_{3,+},\tau_{4,+}]$
\begin{align*}
\frac{\df L_+(\alttheta_+(t))}{\df t} & =- \|\nabla L_+(\alttheta_+)\|_2^2\\
& \leq -\Omega(1) (L_+(\alttheta_+(t))-L_{n,+}).
\end{align*}
By Gr\"onwall inequality, this implies for some $\nu=\Theta(1)$, for any $t\in[\tau_{3,+},\tau_{4,+}]$
\begin{align}
L_+(\alttheta_+(t))-L_{n,+}  & \leq (L_+(\alttheta_+(\tau_{3,+}))-L_{n,+})e^{-\nu (t-\tau_{3,+})}\notag\\
& \leq \frac{\varepsilon_3^2}{2} e^{-\nu (t-\tau_{3,+})}. \label{eq:lossrate}
\end{align}
The last inequality comes from the fact that at the end of the third phase, $\|\hbeta_+(t)-\beta_{n,+}\|_{\Sigma_{n,+}}=\varepsilon_3$.

We bounded by below the norm of $\nabla L_+(\alttheta_+(s))$, but it can also easily be bounded by above as
\begin{align*}
\|\nabla L_+(\alttheta_+(s))\|^2_2 & \leq \left(\sum_{i\in\cI_+}\alta_i(t)^2+\|\altw_i(t)\|_2^2\right)\|D_+(t)\|_2^2\\
& \leq 2\lambda_{\max}(\Sigma_{n,+})\left(\sum_{i\in\cI_+}\alta_i(t)^2\right)\|\hbeta_+(t)-\beta_{n,+}\|_{\Sigma_{n,+}}^2\\
& \leq \bigO{1} (L_+(\alttheta_+(t))-L_{n,+})
\end{align*}
From there, the variation of $\alttheta_+(t)$ can easily be bounded for any $t\in[\tau_{3,+},\tau_{4,+}]$ as
\begin{align}
\|\alttheta_+(t)-\alttheta(\tau_{3,+})\|_2 & \leq \int_{\tau_{3,+}}^t \|\nabla L_+(\alttheta_+(s))\|\df s\notag\\
& \leq \bigO{1} \varepsilon_3\int_{0}^{t-\tau_{3,+}} e^{-\frac{\nu}{2} s}\df s\notag\\
& \leq \bigO{\varepsilon_3}.\label{eq:finitevariation}
\end{align}
Moreover, note that
\begin{align*}
\hbeta_{+}(t)-\hbeta_{\circ}(\tau_{3,+}) & =\sum_{\cI_+}(\alta_i(t)-\alta_i(\tau_{3,+}))\altw_i(\tau_{3,+}) + \sum_{\cI_+}(\altw_i(t)-\altw_i(\tau_{3,+}))\alta_i(\tau_{3,+}).
\end{align*}
In particular,
\begin{align*}
\|\hbeta_{+}(t)-\hbeta_{\circ}(\tau_{3,+})\|_2 & \leq \sum_{\cI_+}|\alta_i(t)-\alta_i(\tau_{3,+})|\|\altw_i(\tau_{3,+})\|_2 + \sum_{\cI_+}\|\altw_i(t)-\altw_i(\tau_{3,+})\|_2\alta_i(\tau_{3,+})\\
&\leq \sqrt{\sum_{\cI_+}(\alta_i(t)-\alta_i(\tau_{3,+}))^2}\sqrt{\sum_{\cI_+}\|\altw_i(\tau_{3,+})\|_2^2} + \sqrt{\sum_{\cI_+}\|\altw_i(t)-\altw_i(\tau_{3,+})\|_2^2}\sqrt{\sum_{\cI_+}\alta_i(\tau_{3,+})^2}\\
&\leq \bigO{1} \|\alttheta(t)-\alttheta(\tau_{3,+})\|_2\\
&\leq\bigO{\varepsilon_3}.
\end{align*}
We can thus choose $\epsstar_3=\bigO{1}$ and $\varepsilon_4=\Theta(\epsstar_3)$ small enough such that \cref{eq:PLaux1} still holds, but $\varepsilon_4$ large enough with respect to $\epsstar_3$ such that the previous inequality ensures for any $t\in[\tau_{3,+},\tau_{4,+}]$:
\begin{equation*}
\|\hbeta_{+}(t)-\hbeta_{\circ}(\tau_{3,+})\|_{\Sigma_{n,+}} \leq \frac{\varepsilon_4}{2}.
\end{equation*}
In particular, this implies that $\tau_{4,+}=+\infty$. Since $\alttheta_+(t)$ has finite variation (Equation~\ref{eq:finitevariation}), this also implies that $\lim_{t\to\infty}\alttheta_+(t)$ exists. The same holds for $\alttheta_-(t)$ by symmetric arguments, so that $\lim_{t\to\infty}\alttheta(t)$ exists. Moreover, \cref{eq:PLaux2,eq:lossrate} imply that
\begin{equation*}
\lim_{t\to\infty}\hbeta_{+}(t)=\beta_{n,+}.
\end{equation*}
This yields the second point of \cref{lemma:phase4}. 

It now remains to prove the first point of \cref{lemma:phase4}. Note that for any $t\geq\tau_{3,+}$ and $i\in\cI_+$:
\begin{align*}
\|\altwn_i(t)-\altwn_i(\tau_{3,+})\|_2 & \leq 2\int_{\tau_{3,+}}^t \|D_+(s)\|_2 \df s \\
&\leq\bigO{\varepsilon_3}.
\end{align*}
Thanks to the first point of \cref{lemma:phase3}, we can choose $\epsstar_3=\Theta(\frac{1}{\sqrt{d}})$ small enough so thatfor any $t\geq\tau_{3,+}$ and $i\in\cI_+$:
\begin{equation*}
 \min_{k\in\cS_{+}}\langle \altwn_i(t), \frac{x_k}{\|x_k\|}\rangle >0 \quad \text{ and }\quad  \max_{k\in\cS_{-}}\langle \altwn_i(t), \frac{x_k}{\|x_k\|}\rangle <0,
\end{equation*}
which concludes the proof of \cref{lemma:phase4}.
\end{proof}

\begin{proof}[Proof of \cref{thm:noconvergence}]
We can conclude the proof of \cref{thm:noconvergence} by noticing that we can indeed choose $\varepsilon,\varepsilon_2,\varepsilon_3,\varepsilon_4$ such that for any $\lambda\leq\lambdastar=\Theta(\frac{1}{d})$ and $n\geq \nstar = \Theta_{\mu}(d^3\log d)$, with probability $1-\bigO{\frac{d^2}{n}+\frac{1}{2^m}}$, the statements of \cref{lemma:phase1,lemma:phase2,lemma:phase3,lemma:phase4} all simultaneously hold. In particular, the stopping times $T_+$ and $T_-$ defined in \cref{lemma:autonomous} are infinite. \cref{lemma:autonomous} then implies that for any $t\geq\tau$, $\alttheta(t)=\theta(t)$. From then, \cref{lemma:phase4} implies \cref{thm:noconvergence}.
\end{proof}

\end{document}